\documentclass[abstract=true]{scrartcl}

\usepackage{amsmath,amssymb,amsthm}
\usepackage{color}

\usepackage[T1]{fontenc}
\usepackage[utf8]{inputenc}
\usepackage{lmodern}
\addtokomafont{disposition}{\rmfamily}
\usepackage[
natbib,
maxcitenames=3, 
maxbibnames=10,  
sorting=ydnt,
url=false,
sortcites,
defernumbers,
backend=biber
]{biblatex}
\usepackage{hyperref}
\hypersetup{
	colorlinks = true,
    final=true,
    plainpages=false,
    pdfstartview=FitV,
    pdftoolbar=true,
    pdfmenubar=true,
    pdfencoding=auto,
    psdextra,
    bookmarksopen=true,
    bookmarksnumbered=true,
    breaklinks=true,
    linktocpage=true
}
\def\setpdfinfo{
	\hypersetup{pdfinfo={
		Title={\@title},
		Author={\@author},
		Subject={\@thesistype},
		Keywords={\@keywords}}}
}

\usepackage[nameinlink, capitalise]{cleveref}
\newtheorem{theorem}{Theorem}[section]
\newtheorem{lemma}[theorem]{Lemma}
\newtheorem{definition}[theorem]{Definition}
\newtheorem{example}[theorem]{Example}
\newtheorem{remark}[theorem]{Remark}
\newtheorem{corollary}[theorem]{Corollary}

\usepackage{xcolor}
\definecolor{faunat}{RGB}{0,155,119}
\definecolor{faulnat}{RGB}{170,207,189}
\definecolor{faullnat}{RGB}{229,239,234}
\colorlet{cfaunat}{faunat>twheel,1,2}
\colorlet{teco}{faunat>twheel,6,12}

\hypersetup{urlcolor=faunat,citecolor=cfaunat,linkcolor=faunat}

\usepackage{todonotes}
\usepackage{cmd/symbols}
\usepackage{enumitem}
\usepackage{kantlipsum}

\title{Continuum Limit of Lipschitz Learning on Graphs}

\author{
Tim Roith
\thanks{Department of Mathematics, University of Erlangen–Nürnberg,
Cauerstraße 11, 91058 Erlangen, Germany. \href{mailto:tim.roith@fau.de}{tim.roith@fau.de}}
\and 
Leon Bungert
\thanks{Hausdorff Center for Mathematics, University of Bonn, Endenicher Allee 62, Villa Maria, 53115 Bonn, Germany. \href{mailto:leon.bungert@hcm.uni-bonn.de}{leon.bungert@hcm.uni-bonn.de}} 
}
\date{\today}

\makeatletter
\let\blx@rerun@biber\relax
\makeatother

\begin{document}

\maketitle
\begin{abstract}
Tackling semi-supervised learning problems with graph-based methods has become a trend in recent years since graphs can represent all kinds of data and provide a suitable framework for studying continuum limits, e.g., of differential operators.
A popular strategy here is $p$-Laplacian learning, which poses a smoothness condition on the sought inference function on the set of unlabeled data.
For $p<\infty$ continuum limits of this approach were studied using tools from $\Gamma$-convergence.
For the case $p=\infty$, which is referred to as Lipschitz learning, continuum limits of the related infinity-Laplacian equation were studied using the concept of viscosity solutions.

In this work, we prove continuum limits of Lipschitz learning using $\Gamma$-convergence. 
In particular, we define a sequence of functionals which approximate the largest local Lipschitz constant of a graph function and prove $\Gamma$-convergence in the $\Lp{\infty}$-topology to the supremum norm of the gradient as the graph becomes denser. 
Furthermore, we show compactness of the functionals which implies convergence of minimizers. 
In our analysis we allow a varying set of labeled data which converges to a general closed set in the Hausdorff distance. 
We apply our results to nonlinear ground states, i.e., minimizers with constrained $L^p$-norm,
and, as a by-product, prove convergence of graph distance functions to geodesic distance functions.

\par\vskip\baselineskip\noindent
\center{\textbf{Keywords}}\\
Lipschitz learning, graph-based semi-supervised learning, continuum limit, Gamma-convergence, ground states, distance functions

\par\vskip\baselineskip\noindent
\center{\textbf{AMS Subject Classification}}\\
\texttt{35J20, 35R02, 65N12, 68T05}

\end{abstract}
\newpage
\tableofcontents

\section{Introduction}
Several works in mathematical data science and machine learning have proven the 
importance of semi-supervised learning as an essential tool for data analysis, see 
\cite{zhu03,zhu2005semi,zhu2005graph,zhu2009introduction,Chap10}. 
Many classification tasks and problems in image analysis 
(see, e.g., \cite{Chap10} for an overview) traditionally require
an expert examining the data by hand, and this so-called labeling process is often a time-consuming and expensive task. 
In contrast, one typically faces an abundance of unlabeled data which one would also like to equip with suitable labels.
This is the key goal of the semi-supervised learning problem which mathematically can be formulated as the extension of a labeling function 
\begin{align*}
g:\consSet\rightarrow\R
\end{align*}
onto the whole data set $V:=\mathcal{V}\cup\consSet$, where $\consSet$ denotes the set of labeled 
and $\mathcal{V}$ the set of unlabeled data.
In most cases, the underlying data can be represented as a finite weighted graph $(V,\omega)$---composed of vertices $V$ and a weight function $\omega$ assigning similarity values to pairs of vertices---which provides a convenient mathematical framework.  
A popular method to generate a unique extension of the labeling function to the whole data set is so called \emph{$p$-Laplacian regularization}, which can be formulated as minimization task
\begin{equation}\label{eq:plapfunc}
\min_{\vecc u:V\to\R} \sum_{x,y\in V} 
\omega(x,y)^p \normR{\vecc u(x)-\vecc u(y)}^p,
\quad\text{ subject to }\vecc u = g\text{ on }\consSet,
\end{equation}
over all graph functions $\vecc u:V\rightarrow\R$ subject to a constraint given by 
the labels on $\consSet$, see, e.g., \cite{Ala16,zhu03,GarcSlep15,Slep19}. 
This method is equivalent to solving the $p$-Laplacian partial differential equations on graphs \cite{Elmo15} and therewith introduces a certain amount of smoothness of the labeling function.
Furthermore, continuum limits of this model as the number of unlabeled data tends to infinity were studied using tools from $\Gamma$-convergence~\cite{GarcSlep15,GarcSlep16,Slep19} and PDEs \cite{Cald19,calder2020poisson,calder2020rates} (see \cref{sec:related_work} for more details).

Still, $p$-Laplacian regularization comes with the drawback that it is ill-posed if $p$ is smaller than the ambient space dimension in the sense that the obtained solutions tend to be an average of the label values rather than properly incorporating the information. 
Extensive studies of this problem were carried out in \cite{GarcSlep16,Slep19}.
To overcome this degeneracy, there are several options: in \cite{calder2020rates} it was investigated at which rates the number of labeled data has to grow to obtain a well-posed problem {for $p=2$ in \labelcref{eq:plapfunc}.}
In \cite{calder2020poisson} it was suggested to replace the pointwise constraint $\vecc u=g$ on $\consSet$ with measure-valued source terms for the graph Laplacian equation.
In contrast, in \cite{Ala16} the authors propose to consider the $p$-Laplacian regularization for large $p$. 
In order to have well-posedness for general space dimensions, one therefore considers the limit $p\rightarrow\infty$ which leads to the so-called Lipschitz learning problem
\begin{equation}\label{eq:LL}
    \min_{\vecc u:V\rightarrow\R} \max_{x,y\in V}\omega(x,y) \normR{\vecc u(x)-\vecc u(y)}
\quad\text{ subject to }\vecc u = g\text{ on }\consSet.
\end{equation}
While in the case $p<\infty$ one has the unique existence of solutions and equivalence of the $p$-Laplacian PDE and the energy minimization task, these properties are lost in the case $p=\infty$. 
One distinguished continuum model{---in the sense that it admits unique solutions---}connected to this problem are \emph{absolutely minimizing Lipschitz extensions} and the associated \emph{infinity Laplacian equation} (see e.g.,~\cite{Aron65,Aron68,aronsson2004tour,juutinen2002absolutely,evans2011everywhere,sheffield2012vector}). 
Using the concept of viscosity solutions, {in \cite{Cald19} a convergence result on continuum limits for the infinity Laplacian equation on the flat torus was established, see again~\cref{sec:related_work} for more details}. 
Still, in \cite{Kyng15} the authors suggest that other Lipschitz extensions (next to the absolutely minimizing) are indeed relevant for machine learning tasks but a rigorous continuum limit for general Lipschitz extensions has been pending.  

The main goal of this paper is to derive a continuum limit for the Lipschitz learning problems~\labelcref{eq:LL} to which end we prove $\Gamma$-convergence and compactness of the functional in~\labelcref{eq:LL}.
We investigate novel smoothness conditions on the underlying domain which are special for this $\Lp{\infty}$-variational problem and originate from the discrepancy between the maximum local Lipschitz constant and the global one.
We apply our results to minimizers of a Rayleigh quotient involving the $\Lp{\infty}$-norm of the gradient as first examined in~\cite{bung20}. {The concrete outline of this paper can be found in 
\cref{sec:outline}.}
\subsection{Assumptions and Main Result}\label{sec:setMain}
Let $\Omega\subset\R^\dims$, $\dims\in\N$, be an open and bounded domain, and let $\domain_n\subset \overline{\domain}$ for $n\in\N$ denote a sequence of finite subsets. For each $n\in\N$ we consider the 
\emph{finite weighted graph} $(\domain_n, \omega_n)$, where
{
$\omega_n:\domain_n\times \domain_n\rightarrow[0,\infty)$} is a weighting function 
which in our context is given as
\begin{align*}
\omega_n(x,y) := \eta_{\scale_n}(\abs{x-y}) := \eta(\abs{x-y}/\scale_n).
\end{align*}
Here $\eta:[0,\infty)\rightarrow[0,\infty)$ denotes the \emph{kernel} and 
$\scale_n>0$ the scaling parameter. 
The edge set of the graph is 
implicitly characterized via the weighting function, i.e., for 
$x,y\in\domain_n$ we have
\begin{equation*}
(x,y)\text{ is an edge iff } \omega_n(x,y)>0.
\end{equation*}
In the following we state standard assumptions on the kernel function 
$\eta$, see, e.g., \cite{GarcSlep15,GarcSlep16,Slep19,Cald19}, 
\begin{enumerate}[label=(K\upshape\arabic*)]
\item\label{en:K1} $\eta$ is positive and continuous at $0$,
\item\label{en:K2} $\eta$ is {non-increasing},
\item\label{en:K4} {$\supp(\eta) \subset [0,\etaradius_\eta]$ for some $\etaradius_\eta>0$}.
\end{enumerate}
Similar to \cite{GarcSlep15} we define the value 
$\sigma_{\eta} := \esssup_{t\geq 0} \left\{\eta(t)~t\right\}$ which is a positive number and appears in the $\Gamma$-limit.

To incorporate constraints, for each $n$ we denote by 
\begin{align*}
\consSet_n\subset\domain_n
\end{align*}
the set of labeled vertices as seen in \labelcref{eq:LL}. Often this set is fixed and therefore independent of $n$ (e.g., \cite{Cald19,GarcSlep15,Slep19}, however, see also \cite{calder2020rates,calder2020poisson} for $p$-Laplacian learning models with varying constraint sets).
As we see in \cref{lem:LIcons} and \cref{lem:LScons} below, making this assumption is not necessary in our case. 
We only require that the sets $\consSet_n$ converge to some closed set $\consSet\subset\overline{\domain}$ in the Hausdorff distance sufficiently fast, i.e., 
\begin{align}\label{eq:labelset_cvgc}
d_H(\consSet_n,\consSet) := \max\left(\sup_{x\in\consSet}\inf_{y\in\consSet_n} |x-y|, \sup_{y\in\consSet_n}\inf_{x\in\consSet} |x-y|\right) = o(\scale_n),\quad n \to \infty,
\end{align}
where $\scale_n$ denotes the spatial scaling of the kernel, introduced above.
A prototypical example for the constraint set is $\consSet=\partial\Omega$ which corresponds to the problem of extending Lipschitz continuous boundary values $g$ from $\partial\Omega$ to $\Omega$.
Considering a labeling function $g:\overline{\domain}\rightarrow\R$ which is Lipschitz 
continuous allows us to restrict it to the finite set of vertices $\consSet_n$.  
The target functional for the discrete case has the form
\begin{align}\label{eq:discrete_func}
\funcd_n(\vecc u) = \frac{1}{\scale_n}\max_{x,y\in\Omega_n}
\left\{\eta_{s_n}(\abs{x-y}) \normR{\vecc u(x) - \vecc u(y)}\right\}
\end{align}
for a function $\vecc u:\domain_n\rightarrow\R$. 

Additionally we define the constrained version of the functional, which 
incorporates the labeling function, as follows
\begin{align*}
\funcd_{n,\mathrm{cons}}(\vecc u)=
\begin{cases}
\funcd_n(\vecc u)&\text{ if } \vecc u = g\text{ on }\consSet_n,\\
\infty&\text{ else.}
\end{cases}
\end{align*}
A typical problem in this context of continuum limits is to find a single metric space in which the convergence of these functionals takes place. 
In our case we choose the normed space $\Lp{\infty}(\domain)$ and thus need to {extend the functionals $\funcd_n$ to $L^{\infty}(\domain)$.} 
This can be achieved by employing the well-established technique (see, e.g., \cite{Bert12}) of only considering piecewise constant functions. 
To this end we let $\vor_n$ denote a closest point projection, i.e., a map  
$\vor_n:\domain\rightarrow\domain_n$ such that
\begin{align*}
\vor_n(x)\in\argmin_{y\in\domain_n}\normR{x-y}
\end{align*}
for each $x\in\domain$. %
{While $\vor_n$ is not necessarily 
uniquely determined, this ambiguity is not relevant for our analysis. There, it is only important to control the value of $\abs{\vor_n(x)-x}$ which is independent of the choice of $\vor_n$.} %
This map has already been employed in \cite{Cald19} and it allows us to transform a graph function 
$\vecc u:\domain_n\to\R$ to a piecewise constant function, by considering $\vecc u\circ\vor_n\in\Lp{\infty}(\domain)$.
This function is constant on each so-called Voronoi cell $\interior{\vor_n^{-1}(y)}$ for $y\in\domain_n$.
This procedure is similar to the technique 
proposed in \cite{GarcSlep15}, where an optimal transport map 
$T_n:\domain\rightarrow\domain_n$ is used for turning graph functions into continuum functions. 
Now we can extend the functional $\funcd_n$ {and $\funcd_{n,\mathrm{cons}}$}
to arbitrary functions $u\in\Lp{\infty}(\domain)$ by defining with a slight abuse of notation 
{
\begin{align}\label{eq:TrafoDisc}
\funcd_{n}(u) 
&:= 
\begin{cases}
\funcd_{n}(\vecc u) 
\quad&\text{if }\exists \vecc u:\domain_n\to\R : u = \vecc u \circ p_n, \\
\infty \quad &\text{else}.
\end{cases}
\\
\label{eq:TrafoDiscConstr}
\funcd_{n,\mathrm{cons}}(u) 
&:= 
\begin{cases}
\funcd_{n,\mathrm{cons}}(\vecc u) 
\quad&\text{if }\exists \vecc u:\domain_n\to\R : u = \vecc u \circ p_n, \\
\infty \quad &\text{else}.
\end{cases}
\end{align}}
In order to control how the discrete sets $\domain_n$ fill out the domain $\domain$, we consider the value 
\begin{align}\label{eq:resolution}
r_n:=\sup_{x\in\domain} \min_{y\in\domain_n} \normR{x-y},
\end{align}
and require that it tends faster to zero than the scaling $\scale_n$, namely, we 
assume that 
\begin{align}\label{eq:scaling}
\frac{r_n}{\scale_n}\longrightarrow 0, \quad n \to \infty.
\end{align}
{
We are interested in the case that $(s_n)_{n\in\N}\subset (0,\infty)$ is a \emph{null sequence} meaning that 
$s_n\rightarrow 0$ for $n\rightarrow \infty$.}
%
\begin{remark}
In the context of continuum limits one often employs random geometric graphs, 
where the discrete sets are obtained as a sequence of points that are 
i.i.d. w.r.t. a probability distribution $\mu\in\MSet(\overline\domain)$.
{
Typically there is no need to use a probabilistic framework in the $L^\infty$ context, since in contrast to the graph $p$-Dirichlet energy \labelcref{eq:Slepfun}, which is a Monte Carlo approximation of an integral functional, the corresponding discrete Lipschitz energy \labelcref{eq:discrete_func} approximates an essential supremum. Therefore, only the support of a probability measure enters our problem.
}
{
Similar observations are made in \cite{Cald19, gruyer2007absolutely} where the 
value $r_n$ is also employed to control the discrete sets~$\domain_n$.}
\end{remark}
{
The $\Gamma$-limit of the discrete functionals $\funcd_n$ 
turns out to be a constant multiple of the following continuum functional}
\begin{align}\label{eq:CFunc}
\func(u) := 
\begin{cases}
\esssup_{x\in\domain}~\normR{\gradR{u}(x)}&\text{if } u \in W^{1,\infty}(\domain),\\
\infty&\text{else.}
\end{cases}
\end{align}%
A constrained version of this functional can be defined analogously
\begin{align}\label{eq:constr_functional}
\func_{\mathrm{cons}}(u) := 
\begin{cases}
\func(u)&\text{if } u \in W^{1,\infty}(\domain)\text{ and } u=g\text{ on }\consSet,\\
\infty&\text{else.}
\end{cases}
\end{align}
Before stating our main results we need to introduce a final assumption on the domain $\domain$ which is necessary because of the discrepancy of the Lipschitz constant and the supremal norm of the gradient of functions on non-convex sets.
For this we introduce the geodesic distance, induced on $\domain$ by the Euclidean distance, which is defined as
\begin{align*}
d_\domain(x,y)=\inf\left\lbrace
\mathrm{len}(\gamma): \gamma:[a,b]\rightarrow\domain 
\text{ is a curve with }
\gamma(a)=x, \gamma(b) = y
\right\rbrace,
\end{align*}
where the length of a curve is given by
\begin{align*}
\mathrm{len}(\gamma):=\sup\left\lbrace\sum_{i=0}^{N-1} \abs{\gamma(t_{i+1})-\gamma(t_i)}\,:\,N\in\N,\,a=t_0\leq t_1\leq\ldots\leq t_N=b\right\rbrace,
\end{align*}
see, e.g., \cite[Prop. 3.2]{Brid99}.
While on convex domains it holds $d_\domain(x,y)=\normR{x-y}$, 
we only need to assume the weaker condition:
\begin{align}\label{eq:cond_domain}
\lim_{\delta\downarrow0}\sup
\left\lbrace
\frac{d_\domain(x,y)}{\normR{x-y}}\,:\,
x,y\in\domain,\,\normR{x-y}<\delta\right\rbrace \leq 1.
\end{align}
In \cref{sec:gamma-cvgc} we explore examples of sets which satisfy 
this condition, however, already at this point we would like to say 
that it is satisfied, for instance, for convex sets, for sets with smooth boundary,
or for sets which locally are diffeomorphic to a convex set.
In a nutshell, condition~\labelcref{eq:cond_domain} prohibits the presence of internal corners in the boundary.
Furthermore, since it holds $d_\domain(x,y)\geq\normR{x-y}$, condition \labelcref{eq:cond_domain} requires the geodesic and the Euclidean distance to coincide locally.
\paragraph{Main results}
Our two main results state the discrete-to-continuum $\Gamma$-con\-ver\-gence of the functionals $\funcd_{n,\mathrm{cons}}$ to $\sigma_\eta~\func_\mathrm{cons}$ and that sequences of minimizers of the discrete functionals converge to a minimizer of the continuum functional.
\begin{theorem}[Discrete to continuum $\Gamma$-convergence]\label{thm:DCGamma}
Let $\domain\subset\Rd$ be a domain satisfying~\labelcref{eq:cond_domain}, let the kernel fulfil \labelcref{en:K1}-\labelcref{en:K4}, and let the constraint sets $\consSet_n,\consSet$ satisfy~\labelcref{eq:labelset_cvgc}, then for any null sequence $\seq{\scale}{n}{(0,\infty)}$ which satisfies the scaling condition~\labelcref{eq:scaling}
we have
\begin{align}
\funcd_{n,\mathrm{cons}}\GConv \sigma_{\eta}~\func_\mathrm{cons}.
\end{align}
\end{theorem}
\begin{theorem}[Convergence of Minimizers]\label{thm:ConvMinLip}
Let $\domain\subset\Rd$ be a domain satisfying~\labelcref{eq:cond_domain}, 
let the kernel fulfil \labelcref{en:K1}-\labelcref{en:K4}, 
let the constraint sets $\consSet_n,\consSet$ satisfy~\labelcref{eq:labelset_cvgc}, 
and $\seq{s}{n}{(0,\infty)}$ be a null sequence which satisfies the scaling 
condition~\labelcref{eq:scaling}.
Then any sequence $\seq{u}{n}{\Lp{\infty}(\Omega)}$ such that 
\begin{align*}
\lim_{n\rightarrow\infty}\left(\funcd_{n,\mathrm{cons}}(u_n) - 
\inf_{u\in\Lp{\infty}(\domain)}\funcd_{n,\mathrm{cons}}(u)\right) = 0
\end{align*}
is relatively compact in $\Lp{\infty}(\domain)$ and 
\begin{align*}
\lim_{n\rightarrow\infty} \funcd_{n,\mathrm{cons}}(u_n) = 
\min_{u\in\Lp{\infty}(\domain)}\sigma_{\eta}~\func_{\mathrm{cons}}(u).
\end{align*}
Furthermore, every cluster point of $(u_n)_{n\in\N}$ is a minimizer of 
$\func_{\mathrm{cons}}$. 
\end{theorem}
\subsection{Related Work}\label{sec:related_work}
{
The first studies concerning the limit behaviour of 
difference operators on random graphs were carried out in 
\cite{Poll81} and the follow up work \cite{Lux08}, 
which considers the consistency of spectral clustering on graphs. 
The main motivation of our paper are the works \cite{GarcSlep15, Slep19} where the $\Gamma$-convergence of the discrete functionals
\begin{align}\label{eq:Slepfun}
\funcd^{(p)}_n(\vecc u) = \frac{1}{\scale_n^p n^2}\sum_{x,y\in\domain_n}
\eta_{\scale_n}(\abs{x-y}) \normR{\vecc u(x)-\vecc u(y)}^p
\end{align}
towards a continuum functional of the form 
\begin{align*}
\func^{(p)}(u) = \sigma_\eta  \int_{\domain} \normR{\gradR{u}}^p \rho^2(x)~\dd x,
\end{align*}
for $1\leq p<\infty$ and $u$ smooth enough, is shown. 
Here, $\rho$ is the density of the 
measure $\mu$ according to which the points $x_i$ are distributed.}
The $\Gamma$-convergence is considered w.r.t. the $\TLP{p}$ topology, 
which allows to compare discrete functions $\vecc u\in\Lp{p}(\nu_n)$ 
with continuum functions $u\in\Lp{p}(\mu)$ 
via an optimal transport ansatz. Here, $\nu_n$ denotes the 
empirical measure for the set $\domain_n$, see \labelcref{eq:empmeasure}.
In particular, the problem they study is connected to the extension 
task \labelcref{eq:plapfunc} by considering the constrained functional
\begin{align*}
\funcd^{(p)}_{n,\mathrm{cons}}(\vecc u):=
\begin{cases}
\funcd^{(p)}_n(\vecc u)&\text{ if } \vecc u(x_i) = g(x_i)
\text{ for }x_i\in\consSet,\\
\infty&\text{ else},
\end{cases}
\end{align*}
and the associated minimization problem
\begin{align}\label{eq:minGarc}
\inf_{\vecc u\in\Lp{p}(\nu_n)}\funcd^{(p)}_{n,\mathrm{cons}}(\vecc u).
\end{align}
Here, the constraint set $\consSet$ is assumed to be a fixed collection of finitely many points.
The main result they show ensures that minimizers of the functionals $\funcd^{(p)}_{n,\mathrm{cons}}$ converge uniformly towards minimizers of a respective constrained version of $\func^{(p)}$ under appropriate assumptions on the scaling $\scale_n$. 
The motivation for considering the limit $p\rightarrow\infty$ for above problems is given in \cite{Ala16}, where the graph $p$-Laplacian for $p<d$ is noticed to have undesirable properties, namely the solution of problem \labelcref{eq:minGarc} tends to %
{
equal a constant at the majority of the vertices 
with sharp spikes at the labeled data.} %
{
Using a suitable scaling,} this problem does not occur in the case $p>\dims$ 
(which is intuitively connected to the Sobolev embedding 
$W^{1,p}(\domain) \hookrightarrow C^{0,1-\dims/p}(\overline\domain)$ for 
$p>\dims$ see \cite[Ch. 4]{Adam03}) and in particular as pointed out in 
\cite[Sec. 3]{Kyng15} one generally hopes for better interpolation properties in the 
case $p\rightarrow\infty$. However, in this limit the interpretation of the 
measure $\mu$ and its density changes{, namely only its support enters the problem.} %
The functional $\func^{(p)}$ incorporates 
information about the distribution of the data in form of the density $\rho$. %
{
Using standard properties of $L^p$-norms, the limit $p\rightarrow\infty$} on the other hand reads
\begin{align*}
\mu\operatorname{-}\esssup_{x\in\domain} \normR{\gradR{u}}=
\lim_{p\rightarrow\infty} \big(\func^{(p)}(u)\big)^{1/p}
\end{align*}
and thus only the support of the measure $\mu$ is relevant. 
This phenomenon was already observed 
in \cite{Ala16, Cald19} which the authors informally described as the 
limit `forgetting' the distribution of the data. 
Furthermore, this observation is consistent with our results, since 
the limit $n\rightarrow\infty$ is independent of 
the sequence $\seq{x}{i}{\domain}$ as long as it fills out the domain 
$\domain$ sufficiently fast, see \cref{thm:DCGamma} for the precise condition.
In the case $p<\infty$ the minimization task \labelcref{eq:plapfunc} is equivalent 
to the graph $p$-Laplace equation, for which the formal limit $p\rightarrow\infty$ 
leads to the graph infinity Laplace equation
\begin{align}\label{eq:GIL}
\begin{split}
\laplace_\infty \vecc u &= 0 \text{ in } V\setminus\consSet,\\
\vecc{u} &= g\text{ in }\consSet,
\end{split}
\end{align}
where the operator $\laplace_\infty$ on $\domain_n$ is defined as
\begin{gather*}
\laplace_\infty\vecc u(x):=\max_{y\in \domain_n}\omega(x,y)(\vecc u(y) - \vecc u(x)) +
\min_{y\in \domain_n}\omega(x,y)(\vecc u(y) - \vecc u(x)).
\end{gather*}
One should note that the unique solution of \labelcref{eq:GIL} also solves the 
Lipschitz learning task \labelcref{eq:LL}, see, e.g., \cite{Cald19,Kyng15}.
{
A first study concerning the continuum limit of this problem is carried out in \cite{Cald19}}. 
{The main result therein states} that the solutions of the discrete problems converge uniformly to the viscosity solution of the continuum equation, 
\begin{align}\label{eq:infLapCont}
\begin{split}
\laplace_\infty u&=0 \text{ in }\domain,\\
u&=g\text{ on }\consSet,
\end{split}
\end{align}
{
where the continuum infinity Laplacian for a smooth function $u$ is defined as
\begin{align}
\Delta_\infty u := \langle \nabla u, D^2 u\nabla u\rangle = \sum_{i,j=1}^d \partial_i u\,\partial_j u\, \partial^2_{ij}u.
\end{align}
}
The considered domain is the flat torus, i.e., $\domain=\Rd\setminus\mathbb{Z}^\dims$ and again the constraint set $\consSet$ is assumed to be fixed and finite.
Furthermore, the only requirement on the sequence of points $\domain_n\subset\domain$ is characterized by the value 
{
$r_n$ defined in \labelcref{eq:resolution}.
}
\begin{theorem}{\cite[Thm. 2.1]{Cald19}}{}
Consider the flat torus $\domain=\Rd\setminus\mathbb{Z}^\dims$, 
let the kernel $\eta\in C^2([0,\infty))$ be such that 
\begin{align*}
\begin{cases}
\eta(s)\geq 1&\text{ for } s\leq 1,\\
\eta(s)=0&\text{ for } s\geq 2,
\end{cases}
\end{align*}
then for a null sequence $\seq{s}{n}{(0,\infty)}$ such that
\begin{align}\label{eq:jeffs_scaling}
\frac{r_n^2}{\scale_n^3}\longrightarrow 0
\end{align}
and for the sequence $\vecc u_n$ of solutions to problem \labelcref{eq:GIL} we have that 
$\vecc u_n \rightarrow u$ uniformly, where $u\in C^{0,1}(\domain)$ is the unique viscosity 
solution of the infinity Laplace equation {\labelcref{eq:infLapCont}} on $\domain$, i.e.,
\begin{align*}
\lim_{n\rightarrow\infty}\max_{x\in\domain_n}\normR{\vecc u_n(x) - u(x)} = 0.
\end{align*}
\end{theorem}
{%
\begin{remark}\label{rem:jeffs_result}
We state this result, since it provides the first continuum limit of the infinity Laplace equation \labelcref{eq:GIL} on general graphs. Solutions of this problem constitute a special subclass of minimizers in the Lipschitz learning task \labelcref{eq:LL}. To show that the limit of solutions of \labelcref{eq:GIL} solve the continuum PDE \labelcref{eq:infLapCont} the author in \cite{Cald19} utilizes a consistency argument which requires smoothness of the kernel $\eta$ and the relatively strict scaling condition \labelcref{eq:jeffs_scaling}. 
In contrast, our results consider the general minimization problem \labelcref{eq:LL}, which allows us to work with the weakest scaling condition possible \labelcref{eq:scaling}---namely that the graph is asymptotically connected---and much weaker conditions on the kernel $\eta$. 
\end{remark}}
{
\begin{remark}[Convergence Rates]
Note that \cite{Cald19} did not establish rates of convergence for solutions of the graph infinity Laplacian equation \labelcref{eq:GIL} and neither do we for general minimizers of \labelcref{eq:LL}.
Indeed, $\Gamma$-convergence is not a good tool for proving quantitative rates since it is a very indirect notion of convergence (see also \cite{GarcSlep16,Slep19} which do not establish convergence rates either).
At the same time $\Gamma$-convergence typically allows for much less restrictive conditions on the graph scaling than PDE techniques, cf.~\cref{rem:jeffs_result}.
Still, in the recent work \cite{bungert2021uniform} we successfully used comparison principle techniques to show rates of convergence for solutions of the graph infinity Laplacian equation \labelcref{eq:GIL} in a much more general setting than the one considered in \cite{Cald19}.
For showing this it suffices to use quantitative versions of the weak scaling assumption \labelcref{eq:scaling} and the domain regularity condition \labelcref{eq:cond_domain}.
\end{remark}
}
{%
Since the solutions of the continuum PDE \labelcref{eq:infLapCont} only exist in the viscosity sense, the proof of this theorem involves viscosity techniques.} The arguments are therefore fundamentally different to the results for the case $p<\infty$ in \cite{GarcSlep16}.
This is our motivation to use $\Gamma$-convergence also for $p=\infty$.
\subsection{Outline}\label{sec:outline}
In \cref{sec:background} we give an overview of the concepts of $\Gamma$-convergence and the closest point projection.
In particular, we derive a transformation rule for supremal functionals which is the analogue of the well-known integral transformation rule 
{
for the change of variables.}

\Cref{sec:gamma-cvgc} is devoted to the proofs of our main results~\cref{thm:DCGamma} and \cref{thm:ConvMinLip}.
Similar to the strategy in \cite{GarcSlep15}, in \cref{sec:nonlocal-to-local} we first prove $\Gamma$-convergence of the non-local auxiliary functionals
\begin{align}\label{eq:NLFunc}
\func_{\scale}(u) := \frac{1}{\scale}~\esssup_{x,y\in\domain}
\left\{
\eta_{\scale}(\abs{x-y})\normR{u(x) - u(y)}
\right\},\quad\scale>0
\end{align}
which mimic the non-local structure of the discrete functionals $\funcd_n$ in \labelcref{eq:discrete_func}, to the continuum functional $\func$ in \eqref{eq:CFunc}.
Subsequently, in \cref{sec:DiscToCont} we use this result for proving our first main result, discrete to continuum $\Gamma$-convergence of the constrained discrete functionals $\funcd_{n,\mathrm{cons}}$.
In \cref{sec:compactness} we prove compactness of the discrete functionals which yields our second main result, the convergence of minimizers.

In \cref{sec:groundstates} we apply our results to a nonlinear eigenvalue problem and prove convergence of discrete ground states to continuum ground states.
Furthermore, generalizing the results in \cite{bung20}, we characterize the latter as geodesic distance function to the constraint set~$\consSet$.

\section{Mathematical Background}\label{sec:background}
This section reviews two important mathematical tools which we  
use in this paper. 
The first one is the concept of $\Gamma$-convergence, 
which allows to deduce convergence of minimizers 
from convergence of functionals.
The second concept, entirely unrelated to $\Gamma$-convergence, is the closest point projection which we employ in order to turn graph functions 
into continuum ones. 
Furthermore, we derive a supremal version of the transformation rule. 
\subsection{\texorpdfstring{$\Gamma$}{}-Convergence}
In this section we introduce a convergence concept that is frequently employed 
in the theory of variational problems, namely the so-called $\Gamma$-convergence. 
We refer to \cite{Brad02} for a detailed introduction. 
\begin{definition}[$\Gamma$-convergence]
Let $X$ be a metric space and let $F_n:X\rightarrow [-\infty,\infty]$ be a sequence of 
functionals. We say that $F_n$ $\Gamma$-converges to the functional 
$F:X\rightarrow [-\infty,\infty]$ if
\begin{enumerate}[label=(\roman*)]
\item \textbf{(liminf inequality)} for every sequence $\seq{x}{n}{X}$ converging to 
$x\in X$ we have that
\begin{align*}
\liminf_{n\rightarrow\infty} F_n(x_n) \geq F(x);
\end{align*}
\item\textbf{(limsup inequality)} for every $x\in X$ there exists a sequence 
$\seq{x}{n}{X}$ converging to $x$ and 
\begin{align*}
\limsup_{n\rightarrow\infty} F_n(x_n)\leq F(x).
\end{align*}
\end{enumerate}
\end{definition}
The notion of $\Gamma$-convergence is especially useful, since it implies 
the convergence of minimizers under additional compactness assumptions. 
For convenience we prove the respective result below, the proof is an 
adaption of a similar result in \cite[Thm. 1.21]{Brad02}.
\begin{lemma}[Convergence of Minimizers]\label{lem:ConvMin}
Let $X$ be a metric space and $F_n:X\rightarrow [0,\infty]$ a sequence of 
functionals $\Gamma$-converging to $F\rightarrow X:[0,\infty]$ which is not 
identically $+\infty$. 
If there exists a relatively compact sequence $(x_n)_{n\in\N}$ such that 
\begin{align*}
\lim_{n\rightarrow\infty} \left(F_n(x_n) - \inf_{x\in X} F_n(x)\right) = 0, 
\end{align*}
then we have that 
\begin{align*}
\lim_{n\rightarrow\infty} \inf_{x\in X} F_n(x) = \min_{x\in X} F(x)
\end{align*}
and any cluster point of $(x_n)_{n\in\N}$ is a minimizer of $F$.
\end{lemma}
\begin{proof}
Using the $\Gamma$-convergence of $F_n$ for any $y\in X$ we can find a sequence 
$\seq{y}{n}{X}$ such that 
\begin{align*}
\limsup_{n\rightarrow\infty} F_n(x_n) =  
\limsup_{n\rightarrow\infty} \inf_{x\in X} F_n(x)\leq 
\limsup_{n\rightarrow\infty} F_n(y_n)\leq
F(y)
\end{align*}
and thus 
\begin{align}\label{eq:liminfproof}
\limsup_{n\rightarrow\infty} F_n(x_n) \leq \inf_{x\in X} F(x)<\infty,
\end{align}
where for the last inequality we use the fact that $F$ is not identically $+\infty$.
By assumption the sequence $(x_n)_{n\in\N}$ is 
relatively compact, therefore we can find an element $x\in X$ and a subsequence such that 
$x_{n_k}\rightarrow x$, 
for which the liminf inequality yields
\begin{align*}
F(x)\leq \liminf_{n\rightarrow\infty} F_n(\tilde x_n)\leq
\liminf_{k\rightarrow\infty} F_{n_k}(x_{n_k})\leq\limsup_{n\rightarrow\infty} F_n(x_{n}),
\end{align*}
where we employ the sequence
\begin{align*}
\tilde x_n :=
\begin{cases}
x_{n_k}&\text{if } n= n_k,\\
x&\text{else.}
\end{cases}
\end{align*}
Together with \labelcref{eq:liminfproof} we have that $x$ is a minimizer of $F$ and 
$\lim_{n\rightarrow\infty} F_n(x_n) = F(x)$. Since the above reasoning works for any subsequence 
converging to some element in $X$ we have that every cluster point is a minimizer. 
\end{proof}
A condition that ensures the existence of a relatively compact 
sequence of minimizers is the so-called \emph{compactness} property for functionals.
A sequence of functionals is called compact if for any sequence $\seq{x}{n}{X}$ the property 
\begin{align*}
\sup_{n\in\N} F_n(x_n) <\infty
\end{align*}
implies that $(x_n)_{n\in\N}$ is relatively compact. 
In \cref{sec:compactness} 
we show that {the constrained functionals $\funcd_{n,\mathrm{cons}}$} fulfill the compactness property. This strategy is 
standard in the context of continuum limits and has already been employed 
in \cite{GarcSlep15,GarcSlep16}.
%
\subsection{The Closest Point Projection}\label{sec:maptn}
In \cite{GarcSlep15,GarcSlep16} a map $T_n:\domain\rightarrow\domain_n$ 
is employed in order to transform integrals w.r.t. the empirical measure, 
defined as
\begin{equation}\label{eq:empmeasure}
\nu_n(A) := \frac{1}{\normR{\domain_n}} \sum_{x\in\domain_n} 
\delta_x(A),\quad A\in \Borel(\domain),
\end{equation}
into integrals w.r.t. a probability measure $\nu\in\MSet(\overline\domain)$. {Here, 
$\mathcal{B}(\domain)$ denotes the Borel $\sigma$-algebra and $\mathcal{P}(\overline{\domain})$ is the set of probability measures on $\domain$.}
One assumes the push-forward condition
\begin{align*}
\nu_n = T_{n\#}\nu = \nu\circ T_n^{-1},
\end{align*}
which yields the following transformation,
\begin{equation*}
\sum_{x\in\domain_n} \vecc u(x) = 
\int_{\domain} \vecc u(x)~\dd\nu_n(x) = \int_{\domain} \vecc u(T_n(x))~\dd\nu(x),
\end{equation*}
for a function $\vecc u:\domain_n\rightarrow\R$, see, for example, \cite{Boga07}. 
Informally speaking the push-forward condition 
manifests the intuition that the map $T_n$ has to preserve the weighting imposed 
by the empirical measure $\nu_n$.
However, the supremal functionals in our case only 
take into account whether a respective set has positive measure or is a null set. 
Therefore, the assumptions on the map $T_n$ can be weakened for an 
analogous transformation rule. In fact, we only need that the 
push-forward measure is equivalent to the original one.
\begin{lemma}\label{lem:suptrafo}
For two probability measures $\mu,\nu \in \MSet(\overline{\domain})$, a measurable map 
$T:\domain\rightarrow\domain$ which fulfills
\begin{enumerate}[label=\upshape(\roman*)]
\item $\nu<<\push{T}{\mu}$,
\item $\push{T}{\mu}<<\nu$,
\end{enumerate}
and for a measurable function $f:\domain\rightarrow\R$ we have that
\begin{align*}
\nu\operatorname{-}\esssup_{x\in\domain} f(x) = \mu\operatorname{-}\esssup_{y\in\domain} f(T(y)).
\end{align*}
\end{lemma}
\begin{remark}{}{}
{
In the case $\nu=\nu_n$ we observe that assumption (i) is equivalent to 
\begin{align}\label{eq:pushgreat}
\mu(T^{-1}(x))>0
\end{align} 
for all $x\in\domain_n$.} 
Furthermore, assumption (ii) is a generalization 
of the property that $T(\domain)\subset\domain_n$. If (i) and (ii) are fulfilled we call 
the measures $\nu$ and $\push{T}{\mu}$ \emph{equivalent}.
{
Additionally, the statement still holds true for a finite measure $\mu$ and a general measure $\nu$. However, for our application it suffices to consider probability measures. 
}
\end{remark}
\begin{proof}
{
First we consider a set $A\in\Borel(\domain)$ such that $\nu(A)=0$. For this we have that
\begin{align*}
\sup_{x\in\domain\setminus A} f(x) \geq \sup_{y\in T^{-1}(\domain\setminus A)} f(T(y)) =
\sup_{y\in \domain\setminus T^{-1}(A)} f(T(y))=(\#)
\end{align*}
and since $\nu(A)=0$ we can use (ii) to infer that $\mu(T^{-1}(A))=0$. This implies that 
\begin{align*}
(\#)\geq \inf_{\mu(B)=0} \sup_{y\in\domain\setminus B} f(T(y)) = 
\mu\operatorname{-}\esssup_{y\in\domain} f(T(y)).
\end{align*}
}
The null set $A$ was arbitrary and thus taking the infimum over all {$\nu$-null} sets we obtain
\begin{equation*}
\nu\operatorname{-}\esssup_{x\in\domain} f(x) \geq \mu\operatorname{-}\esssup_{x\in\domain} f(T(y)).
\end{equation*}
On the other hand take $B\in\Borel(\domain)$ such that $\mu(B)=0$ then 
\begin{align*}
\sup_{y\in\domain\setminus B} f(T(y))= \sup_{x\in T(\domain\setminus B)} f(x)
\end{align*}
and since $T^{-1}(T(\domain\setminus B))\supset \domain\setminus B$ we have that
\begin{align*}
1\geq&\mu(T^{-1}(T(\domain\setminus B)))\geq\mu(\domain\setminus B) = 1\\
\Rightarrow &\mu(\domain\setminus T^{-1}(T(\domain\setminus B)))=0\\
\Rightarrow &\nu(\domain\setminus(T(\domain\setminus B)))=0.
\end{align*}
This implies that
\begin{align*}
\sup_{x\in T(\domain\setminus B)} f(x)\geq 
\inf_{\nu(A)=0}\sup_{x\in \domain\setminus A} f(x) = 
\nu\operatorname{-}\esssup_{x\in\domain} f(x).
\end{align*}
Taking the infimum overall $\mu$-null sets completes the proof.
\end{proof}
An important type of mapping $T_n$ in our context is the so-called closest point 
projection.
\begin{definition}[Closest Point Projection]
For a finite set of points $\domain_n\subset\domain$ a map  
$\vor_n:\domain\rightarrow\domain_n$ is called \emph{closest point projection} if 
\begin{align*}
\vor_n(x)\in\argmin_{y\in\domain_n}\normR{x-y}
\end{align*}
for each $x\in\domain$.
\end{definition}
\begin{remark}{}{}
Recalling the standard definition of a Voronoi tessellation (see, e.g., \cite{Knabner03}) 
one notices that the control volume associated to the vertex $x_i\in\domain_n$ is given by $\interior{\vor_n^{-1}(x_i)}$.
\end{remark}
The use of a closest point projection is very natural for $\Lp{\infty}$-type scenarios and has for example already been employed in \cite{Cald19} for a similar problem. 
In particular we can see that
\begin{align*}
\lambda^{\dims}(\vor_n^{-1}(x_i)) >0
\end{align*}
for every vertex $x_i\in\domain_n$ and thus $\nu_n<<\vor_{n\#}\lambda^{\dims}$, where 
$\lambda^\dims$ denotes the $\dims$-dimensional Lebesgue measure. 
The second condition $\vor_{n\#}\lambda^{\dims}<<\nu_n$ follows directly from 
the definition of the map $\vor_n$ {and thus the conditions for \cref{lem:suptrafo} are fulfilled}. %
{
In fact, for each function $u\in\Lp{\infty}(\domain)$ such that $u = \vecc u \circ p_n$ for some $\vecc u:\domain_n\to\R$ we can employ
\cref{lem:suptrafo} to reformulate the extension \labelcref{eq:TrafoDisc} of the discrete functional as follows, 
\begin{align}
    \notag
    \funcd_n(u) &= \funcd_n(\vecc u)\\
    \notag
    &= \frac{1}{\scale_n}\max_{x,y\in\domain_n} 
    \eta_{\scale_n}(\abs{x-y}) \abs{\vecc u(x)-\vecc u(y)}  \\
    \notag
    &= \frac{1}{\scale_n}
    \nu_n\operatorname{-}\esssup_{x,y\in\domain}
    \eta_{\scale_n}(\abs{x-y}) \abs{\vecc u(x)-\vecc u(y)}  \\
    \label{eq:trafo_discrete_func}
    &=\frac{1}{\scale_n}\esssup_{x,y\in\domain} 
    \eta_{\scale_n}(\abs{\vor_n(x)-\vor_n(y)}) \abs{u(x)-u(y)}.
\end{align}}
Note that the {weights} consider the distance between the nearest vertices to $x$ and $y$, respectively, and not the Euclidean distance between $x$ and $y$. This observation is important for 
the estimate in \cref{sec:DiscToCont}.
\section{\texorpdfstring{$\Gamma$}{}-Convergence of Lipschitz Functionals}\label{sec:gamma-cvgc}
\subsection{Non-Local to Local Convergence}\label{sec:nonlocal-to-local}
In this section we show the $\Gamma$-convergence of the non-local functionals 
\labelcref{eq:NLFunc} to the continuum functional defined in \labelcref{eq:CFunc} 
with respect to the $\Lp{\infty}$ topology. 
We first prove the liminf inequality.
\begin{lemma}[liminf inequality]\label{lem:NLLliminf}
Let $\domain\subset\Rd$ be an open domain and let the kernel fulfil 
\labelcref{en:K1}-\labelcref{en:K4}, 
then for a null sequence $\seq{s}{n}{(0,\infty)}$ we have 
\begin{align}
\liminf_{n\rightarrow\infty}\func_{\scale_n}(u_n) \geq
\sigma_{\eta}~\func(u)
\end{align}
for every sequence $\seq{u}{n}{\Lp{\infty}(\Omega)}$ converging 
to $u\in \Lp{\infty}(\Omega)$ in $\Lp{\infty}(\Omega)$.
\end{lemma}
\begin{proof}
We assume w.l.o.g. that 
\begin{equation}\label{eq:energybound}
\liminf_{n\rightarrow\infty}\func_{\scale_n}(u_n) < \infty.
\end{equation}
We choose a vector $h \in \Rd$ and estimate the supremum over $x,y\in\domain$ by a 
supremum over a difference quotient, namely
\begin{align*}
\func_{\scale_n}(u_n) &=
\esssup_{x,y\in\domain} \eta_{\scale_n}(\abs{x-y})~
\frac{\normR{u_n(x) - u_n(y)}}{\scale_n}
\\&\geq
\eta(\abs{h})~\esssup_{x\in\domain}
\frac{\normR{u_n(x) - u_n(x + \scale_n h)}}{\scale_n}~\ind_\domain(x+\scale_n h).
\end{align*}
In the above transformation we ensured to not enlarge the supremum by multiplying by the 
indicator function. Considering the function
\begin{align*}
v_n^h(x):= \frac{u_n(x) - u_n(x + \scale_n h)}{\scale_n}~\ind_\domain(x+\scale_n h)
\end{align*}
for $\eta(\abs{h})>0$ we have that 
\begin{equation*}
\liminf_{n\rightarrow\infty}\norm{v_n^h}_{\Lp{\infty}}\leq C
\end{equation*}
{which follows directly form \labelcref{eq:energybound}.}
Thus, by the sequential Banach--Alaoglu theorem, the sequence $(v_n^h)_{n\in\N}$ possesses convergent subsequences.
For any such subsequence $(v_{n_k}^h)_{k\in\N}$ there exists $v^h\in\Lp{\infty}$ such that
\begin{equation*}
v_{n_k}^h \rightharpoonup^\ast v^h
\end{equation*}
in the \weaks topology of $\Lp{\infty}$, 
i.e., for every $w\in\Lp{1}(\domain)$ we have 
\begin{align}\label{eq:weaksconv}
\int_{\domain} v_{n_k}^h~w~\dd x \rightarrow \int_{\domain} v^h~w~\dd x.
\end{align}
We want to identify the function $v^h$, for which we use a smooth 
function $\phi\in C_c^\infty(\domain)$ as the test function in \labelcref{eq:weaksconv}.
We shift the difference quotient of $u_n$ to a quotient of $\phi$ and hope to 
obtain the directional derivative in the limit. 
Since $\supp(\phi)\subset\subset\domain$, we can choose $n_0$ large enough such that 
\begin{align*}
x+\scale_n h\in\domain 
\end{align*}
for all $n\geq n_0$ and for all $x\in\supp(\phi)$.
{
Therefore, we get
\begin{align*}
\int_{\domain} v_{n_k}^h(x)~\phi(x)~\dd x&=
\int_{\domain} \frac{u_n(x) - u_n(x + \scale_n h(x))}{\scale_n}~\phi(x)~\dd x\\&=
\int_{\domain} u_n(x)~
\frac{\phi(x) - \phi(x - \scale_n h)}{\scale_n} ~\dd x.
\end{align*}}
Furthermore, for $n\geq n_0$ we have
\begin{align*}
&\normR{u_n(x)~\frac{\phi(x) - \phi(x - \scale_n h)}{\scale_n} 
- u(x) \gradR{\phi}(x)\cdot (-h)}\\ \leq
&\norm{u_n-u}_{\Lp{\infty}}~\normR{\frac{\phi(x) - \phi(x - \scale_n h)}{\scale_n}}\\&+
\norm{u}_{\Lp{\infty}}
\normR{\frac{\phi(x) - \phi(x - \scale_n h)}{\scale_n} -\gradR{\phi}(x)\cdot (-h)}\\ \leq
&\normR{h}\norm{u_n-u}_{\Lp{\infty}}\norm{\gradR{\phi}}_{\Lp{\infty}}\\&+
\norm{u}_{\Lp{\infty}}
\normR{\frac{\phi(x) - \phi(x - \scale_n h)+\gradR{\phi}(x)\cdot (s_nh)}{\scale_n}}
\xrightarrow{n\rightarrow\infty} 0,
\end{align*}
since $u_n$ converges to $u$ in $\Lp{\infty}$, 
$\phi\in C^\infty_c$ has a bounded gradient, and since the difference quotient converges 
to the directional derivative. Besides the pointwise convergence, we also easily obtain 
the boundedness of the function sequence, since
\begin{align*}
\normR{u_n(x)~\frac{\phi(x) - \phi(x - \scale_n h)}{\scale_n}}&\leq
\normR{h}\norm{u_n}_{\Lp{\infty}}~\norm{\gradR{\phi}}_{\Lp{\infty}}\\ &\leq
\normR{h}(\norm{u_n-u}_{\Lp{\infty}}+\norm{u}_{\Lp{\infty}})~\norm{\gradR{\phi}}_{\Lp{\infty}},
\end{align*}
which is uniformly bounded.
Thus, we can apply Lebesgue's convergence 
theorem to see that
\begin{align*}
\int_{\Rd} v^h~\phi~\dd x = 
\lim_{k\rightarrow\infty}\int_{\domain} v_{n_k}^h~\phi~\dd x =
-\int_{\Rd} u~(\gradR{\phi}\cdot h)~\dd x.
\end{align*}
In particular, we can choose $h_i=c\,e_i$, {where $e_i$ denotes the $i$-th unit vector and} the constant $c>0$ is small enough to ensure that $\eta(\abs{h_i})>0$, to obtain
\begin{align*}
\int_{\Rd} v^{h_i}~\phi~\dd x &= -c\int_{\Rd} u~\partial_i \phi~\dd x
\end{align*}
for all $i\in\{1,\ldots,\dims\}$ and all $\phi\in C^\infty_c(\domain)$.
This yields that $u\in W^{1,\infty}(\domain)$ and again for  {any} $h$ such that $\eta(\abs{h})>0$
\begin{align*}
\int_{\Rd} v^h~\phi~\dd x = \int_{\Rd} (\gradR{u}\cdot h)~\phi~\dd x.
\end{align*}
Using the density of $C^\infty_c(\domain)$ in $\Lp{1}(\domain)$ w.r.t. $\norm{\cdot}_{\Lp{1}}$ we obtain that 
\begin{align*}
\int_{\Rd} v^h~w~\dd x = \int_{\Rd} (\gradR{u}\cdot h)~w~\dd x
\end{align*}
for any $w\in\Lp{1}(\domain)$. 
Since the limit is independent of the subsequence $v_{n_k}^h$, we obtain 
that the \weaks convergence holds for the whole sequence, i.e., $v_{n}^h \rightharpoonup^\ast \nabla u\cdot h$ and thus together with the lower semi-continuity of $\norm{\cdot}_{\Lp{\infty}}$
\begin{align*}
\liminf_{n\rightarrow\infty}\func_{\scale_n}(u_n) \geq
\eta(\abs{h})~\liminf_{n\rightarrow\infty}\norm{v_{n}^h}_{\Lp{\infty}} \geq
\eta(\abs{h})~\norm{\gradR{u}\cdot h}_{\Lp{\infty}},
\end{align*}
{for every $h\in\R^d$ such that $\eta(\abs{h})>0$.
Since the inequality is trivially true for $\eta(\abs{h})=0$ we obtain
\begin{align*}
\liminf_{n\rightarrow\infty}\func_{\scale_n}(u_n) \geq 
\sup_{h\in\Rd}\eta(\abs{h})\norm{\gradR{u}\cdot h}_{\Lp{\infty}}.
\end{align*}
Considering $z\in\domain$ such that $\nabla u(z)$ exists and satisfies 
$\normR{\gradR{u}(z)}>0$, and taking $t\geq0$ we have that
\begin{align*}
\sup_{h\in\Rd}\eta(\abs{h})~\norm{\gradR{u}\cdot h}_{\Lp{\infty}} \geq
\eta(t)\,
\norm{\gradR{u}\cdot t\frac{\gradR{u}(z)}
{\normR{\gradR{u}(z)}}}_{\Lp{\infty}}\geq
\eta(t)\,t\,{\normR{\gradR{u}(z)}}.
\end{align*}
}%
{This inequality holds for every $t\geq 0$ 
and almost every $z\in\domain$, since it is again trivially fulfilled if $\nabla u(z)$ exists and is equal to zero. 
Hence, we obtain}
\begin{align*}
\liminf_{n\rightarrow\infty}\func_{\scale_n}(u_n) 
\geq 
\sigma_\eta~\func(u)
\end{align*}
which concludes the proof.
\end{proof}
We proceed by proving the limsup inequality. The most important fact here 
is that for $u\in W^{1,\infty}(\Omega)$ and for almost every $x,y\in\domain$ 
we have the inequality
\begin{align}\label{eq:lipeq}
\normR{u(x)-u(y)}\leq\norm{\gradR{u}}_{\Lp{\infty}}
d_{\domain}(x,y),
\end{align}
where $d_{\domain}(\cdot,\cdot)$ denotes the geodesic distance on $\domain$, 
see~\cite[P. 269]{Brez10}. 
Since the non-local functional $\func_s$ compares points $x,y\in\domain$ that are close together w.r.t. the Euclidean distance, we need to asymptotically bound the geodesic distance from above by the Euclidean distance. 
For this, we assume condition~\labelcref{eq:cond_domain}, which we repeat here for convenience:
\begin{align*}
\lim_{\delta\downarrow0}\sup\left\lbrace\frac{d_\domain(x,y)}{\normR{x-y}}\,:\,x,y\in\domain,\,\normR{x-y}<\delta\right\rbrace \leq 1.
\end{align*}
\begin{lemma}[limsup inequality]\label{NLCGammasup}
Let $\domain\subset\Rd$ be a domain satisfying~\labelcref{eq:cond_domain}, 
$\seq{s}{n}{(0,\infty)}$ a null sequence and let the kernel fulfil 
\labelcref{en:K1}-\labelcref{en:K4}, then for each 
$u\in \Lp{\infty}(\Omega)$ there exists a sequence 
$\seq{u}{n}{\Lp{\infty}(\Omega)}$ converging to $u$ 
strongly in $\Lp{\infty}(\Omega)$ such that
\begin{align}
\limsup_{n\rightarrow\infty}\func_{\scale_n}(u_n) \leq
\sigma_\eta~\func(u).
\end{align}
\end{lemma}
\begin{proof}
If $u\notin W^{1,\infty}$ the inequality holds trivially. 
If $u\in W^{1,\infty}$ we see that 
\begin{align*}
\func_{\scale_n}(u) &= \frac{1}{\scale_n}~\esssup_{x,y\in\domain}
\left\{
\eta_{\scale_n}(\abs{x-y})\normR{u(x) - u(y)}\right\}\\ &\leq
\frac{1}{\scale_n}~\esssup_{x,y\in\domain}
\left\{
\eta_{\scale_n}(\abs{x-y})~d_{\domain}(x,y)\right\}~
\norm{\gradR{u}}_{\Lp{\infty}} \\ &\leq
\frac{1}{\scale_n}~\esssup_{x,y\in\domain}
\left\{
\eta_{\scale_n}(\abs{x-y})~
\normR{x-y}\frac{d_\domain(x,y)}{\normR{x-y}}\right\}~\norm{\gradR{u}}_{\Lp{\infty}}.
\end{align*}
By \labelcref{eq:cond_domain}, for any $\varepsilon>0$ we can find $\delta>0$ such that 
\begin{align*}
    \frac{d_\domain(x,y)}{\normR{x-y}}\leq 1+\varepsilon,\quad\forall x,y\in\domain\,:\,\normR{x-y}<\delta.
\end{align*}
Choosing $n\in\N$ so large that $\etaradius_\eta\scale_n<\delta$, where $\etaradius_\eta$ is the radius of the kernel $\eta$, we obtain
\begin{align*}
\func_{\scale_n}(u) &\leq 
(1+\varepsilon)\esssup_{x,y\in\domain}
\left\{
\eta_{\scale_n} (\abs{x-y})\frac{\abs{x-y}}{s_n}
\right\}
~\norm{\gradR{u}}_{\Lp{\infty}} \\ &=
(1+\varepsilon)\esssup_{z\in\Rd}
\left\{
\eta(\abs{z})\normR{z}
\right\}
~\norm{\gradR{u}}_{\Lp{\infty}} \\ &=
(1+\varepsilon)~\sigma_\eta~\norm{\gradR{u}}_{\Lp{\infty}} \\ &=
(1+\varepsilon)~{\sigma_\eta}\,\func(u).
\end{align*}
Since, $\varepsilon>0$ was arbitrary, this shows that the constant sequence $u_n:=u$ fulfills the limsup inequality.
\end{proof}
The previous lemmata directly imply the $\Gamma$-convergence of the 
respective functionals, which we state below.
\begin{theorem}[Non-local to local $\Gamma$-convergence]\label{thm:NLCGamma}
Let $\domain\subset\Rd$ be a domain satisfying~\labelcref{eq:cond_domain} and let the kernel fulfil 
\labelcref{en:K1}-\labelcref{en:K4}, then for any null sequence $\seq{s}{n}{(0,\infty)}$ we have that
\begin{align}
\func_{\scale_n}\GConv \sigma_{\eta}~\func.
\end{align}
\end{theorem}
\begin{remark}
Assumption~\labelcref{eq:cond_domain} is not satisfied for general non-convex domains, whereas 
\begin{align*}
    \normR{u(x)-u(y)}\leq\Lip(u)\normR{x-y}
\end{align*}
is.
Hence, one might consider replacing the functional $\func(u)=\sigma_\eta\norm{\gradR{u}}_{\Lp{\infty}}$ by $\func(u)=\sigma_\eta\Lip(u)$ which allows to prove the limsup inequality for arbitrary (in particular non-convex) domains.
However, as the following example shows, the liminf inequality is not true for this functional and one has
\begin{align*}
    \sigma_\eta\norm{\gradR{u}}_{\Lp{\infty}}\leq\liminf_{n\to\infty}\func_{\scale_n}(u_n)\leq\limsup_{n\to\infty}\func_{\scale_n}(u_n) \leq \sigma_\eta \Lip(u)
\end{align*}
where each inequality can be {strict}.
\end{remark}
\begin{example}
We consider the non-convex domain $\domain=\{x\in\R^2 \;:\, \max(|x_1|,|x_2|)\\{<} 1\}\setminus\left([0,1]\times [-1,0]\right)$ which does not satisfy~\labelcref{eq:cond_domain}, the function
\begin{align*}
    u(x)=
    \begin{cases}
    x_1^p \quad&\text{if } x_1,x_2 \geq 0,\\
    x_2^p \quad&\text{if } x_1,x_2 \leq 0,\\
    0 \quad&\text{else},
    \end{cases}
\end{align*}
for some power $p\geq 1$ and the kernel $\eta(t)=\chi_{[0,1]}(t)$ for $x\geq 0$.
Then one can compute that
\begin{align*}
    \norm{\gradR{u}}_{\Lp{\infty}} =p ,\quad
    \Lip(u) = \max(\sqrt{2},p), \quad
    \lim_{n\to\infty}\func_{\scale_n}(u) = 
    \begin{cases}
    \sqrt{2},\quad &p=1,\\
    p,\quad &p>1.
    \end{cases}
\end{align*}
The case $1<p<\sqrt{2}$ shows that the liminf inequality $\liminf_{n\to\infty}\func_{\scale_n}(u_n)\geq \Lip(u)$ is false, in general.
\end{example}
\begin{example}[The domain condition~\labelcref{eq:cond_domain}]
In this example we will study several scenarios where condition~\labelcref{eq:cond_domain} is satisfied. 
Let us first remark that if one fixes $x\in\domain$ then 
\begin{align*}
\lim_{\delta\downarrow0}\sup\left\lbrace\frac{d_\domain(x,y)}{\normR{x-y}}\,:\,y\in\domain,\,\normR{x-y}<{\delta}\right\rbrace \leq 1.
\end{align*}
is always true since $\domain$ is open.
Hence, \labelcref{eq:cond_domain} is in fact a condition on the boundary of the domain.
\begin{itemize}
    \item If $\domain$ is convex, it holds $d_\domain(x,y)=\normR{x-y}$ and hence \labelcref{eq:cond_domain} is trivially true.
    \item {If $\domain$ is locally $C^{1,1}$-diffeomorphic to a convex set, then \labelcref{eq:cond_domain} is satisfied as well. By this we mean that for all $x\in\domain$ there exists $\delta>0$, a convex set $C\subset\R^d$, and a diffeomorphism $\Phi:C\to\R^d$ with inverse $\Psi$ such that $B_\delta(x)\cap \domain = \Phi(C)$.}
    In particular, this includes domains with a sufficiently regular boundary.
    {To see this let $x\in\domain$ and $y\in B_\delta(x)\cap\domain=\Phi(C)$.}
    Because $C$ is convex, we can connect $\Psi(x)$ and $\Psi(y)$ with a straight line $\tau(t)=(1-t)\Psi(x)+t\Psi(y)$ for $t\in[0,1]$ and consider the curve {$\gamma(t)=\Phi(\tau(t))\subset\Phi(C)$ which lies in $B_\delta(x)\cap\domain$ since $\tau$ lies in $C$}.
    Hence,
    \begin{align*}
        d_\domain(x,y)
        &\leq \mathrm{len}(\gamma) \\
        &= \int_0^1\normR{\dot{\gamma}(t)}\d t \\
        &\leq \int_0^1 \normR{\gradR{\Phi}(\tau(t))}\normR{\dot{\tau}(t)}\d t \\
        &= \int_0^1 \normR{\gradR{\Phi}(\tau(t))}\normR{\Psi(y)-\Psi(x)}\d t \\
        &= \int_0^1 \normR{\gradR{\Phi}(\tau(t))} \normR{\gradR{\Psi}(x)}\normR{x-y} \d t + o(\normR{x-y}) \\
        &\leq \int_0^1 \normR{\gradR{\Phi}(\tau(t))} \normR{\gradR{\Psi}(\gamma(t))}\normR{x-y} \d t + \\
        &\qquad \int_0^1 \normR{\gradR{\Phi}(\tau(t))} \normR{\gradR{\Psi}(\gamma(t))-\gradR{\Psi}(x)}\normR{x-y} \d t 
        + o(\normR{x-y}) \\
        &\leq \normR{x-y} + c\normR{x-y} \int_0^1 \normR{\gradR{\Phi}(\tau(t))} \normR{\gamma(t)-x}\d t + o(\normR{x-y}) \\
        &= \normR{x-y} + c\normR{x-y} \int_0^1 \normR{\gradR{\Phi}(\tau(t))} \normR{\Phi(\tau(t))-\Phi(\Psi(x))}\d t + o(\normR{x-y}) \\
        &\leq \normR{x-y} + c\normR{x-y} \int_0^1 \normR{\tau(t)-\Psi(x)}\d t + o(\normR{x-y}) \\
        &= \normR{x-y} + c\normR{x-y} \normR{\Psi(y)-\Psi(x)} \int_0^1 t\d t + o(\normR{x-y}) \\
        &\leq \normR{x-y} + c\normR{x-y}^2 + o(\normR{x-y}),
    \end{align*}
    where we used Lipschitz continuity of $\Phi,\Psi$ and $\gradR{\Psi}$ {and the fact that $\Phi$ is a diffeomorphism which implies $\nabla\Phi(\tau(t))=(\nabla\Psi(\gamma(t)))^{-1}$.} 
    Note that the constant $c>0$ is changing with every inequality.
    Dividing by $\normR{x-y}$ and letting $\normR{x-y}\to 0$, we finally get~\labelcref{eq:cond_domain}.
\end{itemize}
\end{example}
%
%
\subsection{Discrete to Continuum Convergence}\label{sec:DiscToCont}
We now consider the $\Gamma$-convergence of the discrete functionals. While in the
previous section we employed an arbitrary null sequence 
$\seq{\scale}{n}{(0,\infty)}$ for the scaling, 
we are now limited to certain scaling sequences depending on the the sequence 
of {sets} $\domain_n$.
In particular, we have to control how fast the scaling 
$\scale_n$ tends to zero in comparison to how fast the points in $\domain_n$
fill out the domain $\domain$. 
The following simple example illustrates why we have to consider 
the relationship between $\scale_n$ and $\domain_n$.
\begin{example}\label{ex:fuldisc}
Let $\seq{x}{n}{\Rd}$ be an arbitrary sequence of points, then we can choose $\scale_n$ 
small enough such that $\eta_{\scale_n}(\abs{x-y}) =0$ for $x,y \in\domain_n$ and 
thus we have that $\funcd_n(u_n)=0$ for every $n\in\N$. 
In this situation the liminf inequality does not hold true. 
\end{example}
As illustrated in the example above, we need to take special care of points 
$x,y\in\domain_n$, where $\eta_{\scale_n}(\abs{x-y})=0$. 
Formulating this problem in terms of the 
map $\vor_n$ we have to consider the case where
\begin{align*}
\eta_{\scale_n}(\abs{\vor_n(x)-\vor_n(y)})=0.
\end{align*}
{Using that the kernel has radius $\etaradius_\eta<\infty$ it follows that $\normR{\vor_n(x)-\vor_n(y)}>\etaradius_\eta \scale_n$ and thus
\begin{align*}
\begin{split}
\normR{x-y}&= \normR{x-\vor_n(x)+\vor_n(x)-\vor_n(y)+\vor_n(y)-y}\\&\geq 
\normR{\vor_n(x)-\vor_n(y)}-2\norm{\id - \vor_n}_{\Lp{\infty}}\\ &>
\etaradius_\eta\scale_n-2\norm{\id - \vor_n}_{\Lp{\infty}}=:\etaradius_\eta\tilde{\scale}_n.
\end{split}
\end{align*}}
The idea now is to use this new scaling $\tilde{\scale}_n$ for the non-local functionals, 
where we have to impose that $\tilde{\scale}_n >0$ for all $n$ large enough. But more 
importantly we must ensure that the quotient $\tilde{\scale}_n/\scale_n$ converges to $1$, 
i.e., 
{%
\begin{align*}
\frac{\tilde\scale_n}{\scale_n}=
\frac{\scale_n-2\norm{\id - \vor_n}_{\Lp{\infty}}/\etaradius_\eta}{\scale_n}= 
1-\frac{2\norm{\id - \vor_n}_{\Lp{\infty}}/\etaradius_\eta}{\scale_n} \longrightarrow 1,
\end{align*}}
which is equivalent to the the fact that 
\begin{align*}
\frac{\norm{\id - \vor_n}_{\Lp{\infty}}}{\scale_n}\longrightarrow 0.
\end{align*}
This argumentation was first applied in \cite{GarcSlep15}, where instead of the map 
$\vor_n$ an optimal transport map $T_n$ was employed. 
For a closest point projection $\vor_n$ we know that 
\begin{align*}
\norm{\id - \vor_n}_{\Lp{\infty}} = \sup_{x\in\domain} \dist(x, \domain_n)=r_n.
\end{align*} 
which thus yields the scaling assumption \labelcref{eq:scaling}.
\begin{lemma}[liminf inequality]\label{lem:LIcons}
Let $\domain\subset\Rd$ be a domain, let the constraint sets satisfy~\labelcref{eq:labelset_cvgc}, and let the kernel fulfil 
\labelcref{en:K1}-\labelcref{en:K4}, then for any null sequence $\seq{\scale}{n}{(0,\infty)}$ which satisfies the scaling condition~\labelcref{eq:scaling} we have that
\begin{align*}
\liminf_{n\rightarrow\infty}\funcd_{n,\mathrm{cons}}(u_n) \geq \sigma_{\eta}~\func_\mathrm{cons}(u)
\end{align*}
for every sequence $\seq{u}{n}{\Lp{\infty}(\Omega)}$ converging 
to $u\in \Lp{\infty}(\Omega)$ in $\Lp{\infty}(\Omega)$.
\end{lemma}
\begin{proof}
W.l.o.g.~we assume that $\liminf_{n\rightarrow\infty}\funcd_{n,\mathrm{cons}}(u_n) <\infty$.
After possibly passing to a subsequence, we can furthermore assume that $u_n=g$ on $\consSet_n$.
We first show that the limit function $u$ satisfies $u=g$ on $\consSet$.

Since $\eta$ is continuous and positive in $0$, we know that there exists $0<t<\etaradius_\eta$ such that $\eta(s)>C$ for all $0<s<t$ where $C>0$.
Furthermore, using \labelcref{eq:labelset_cvgc} we infer that for all $x\in\consSet$ there exists $x_n\in\consSet_n$ with $|x-x_n|=o(\scale_n)$.
In particular, for $n$ large enough it holds $\normR{x-x_n}\leq s_nt$.
This allows us to estimate:
\begin{align*}
    \normR{u(x)-g(x)} 
    &\leq \normR{u(x)-u_n(x)} + \normR{u_n(x)-u_n(x_n)} \\
    &\qquad \qquad + 
    \normR{u_n(x_n)-g(x_n)} + \normR{g(x_n) - g(x)} \\
    &\leq \norm{u-u_n}_{\Lp{\infty}(\domain)} + 
    \frac{1}{C}\eta(\abs{x-x_n})\normR{u_n(x)-u_n({x_n})} \\
    &\qquad \qquad + 0 + \Lip(g)\normR{x-x_n} \\
    &\leq \norm{u-u_n}_{\Lp{\infty}(\domain)} + 
    \frac{\scale_n}{C} \funcd_{n,\mathrm{cons}}(u_n) + \Lip(g)\normR{x-x_n} \\
    &\leq \norm{u-u_n}_{\Lp{\infty}(\domain)} + 
    \frac{\scale_n}{C} \funcd_{n,\mathrm{cons}}(u_n) + \Lip(g) s_n\,{t}.
\end{align*}
Taking $n\to\infty$, using that $\funcd_{\scale_n}(u_n)$ is uniformly bounded and $\scale_n\to 0$, we obtain $u=g$ on $\consSet$.

The main idea for proving the liminf inequality here is to establish a discrete to non-local control estimate and then use~\cref{lem:NLLliminf}.

Since we assumed $\liminf_{n\rightarrow\infty}\funcd_{n,\mathrm{cons}}(u_n) <\infty$, we know that $u_n$ is piecewise constant for every $n\in\N$, in the sense of \cref{sec:setMain}, i.e., $u_n=\vecc u_n \circ p_n$ for some $\vecc u_n:\domain_n\to\domain$.
As seen in \labelcref{eq:trafo_discrete_func} we can express $\funcd_{n,\mathrm{cons}}$ as follows:
\begin{align*}
\funcd_{n,\mathrm{cons}}(u_n) 
=\frac{1}{\scale_n}\esssup_{x,y \in \domain}\eta_{\scale_n}(\abs{\vor_n(x)-\vor_n(y)}) \normR{u_n(x) - u_n(y)}=(\#).
\end{align*}
In order to apply \cref{lem:NLLliminf}, we need to transform the {weighting} that considers 
the distance between $\vor_n(x)$ and $\vor_n(y)$ into another one that measures the distance between $x$ and $y$.\\
\textbf{Case 1:} There exists $t>0$ such that $\eta$ is constant on {$[0,t]$}.\\
We employ the observation that whenever $\normR{\vor_n(x)-\vor_n(y)}>\scale_n\,t$, {
for the new scaling 
$\tilde{\scale}_n:=\scale_n-2r_n/t$ we have}
\begin{align*}
\frac{\normR{x-y}}{\tilde{\scale}_n} &\geq
\frac{\normR{\vor_n(x)-\vor_n(y)} -2 r_n}{\scale_n-2r_n/t}\\&= 
\frac{\normR{\vor_n(x)-\vor_n(y)}}{\scale_n}
\underbrace{\frac{1 -2 r_n/\normR{\vor_n(x)-\vor_n(y)}}{1-2r_n/(t\scale_n)}}_{>1}\\&>
\frac{\normR{\vor_n(x)-\vor_n(y)}}{\scale_n},
\end{align*}
{where we used that $r_n=\norm{\id-\vor_n}_{\Lp{\infty}}$.}
Since $\eta$ {is non-increasing \labelcref{en:K2} and $\eta$ is constant on $[0,t)$}, we get
\begin{align*}
\eta_{\scale_n}(\abs{\vor_n(x)-\vor_n(y)}) \geq
\eta_{\tilde{\scale}_n}(\abs{x-y})
\end{align*}
{for almost all $x,y\in\domain$.}
This allows us to further estimate
\begin{align*}
(\#) \geq \frac{1}{\scale_n}\esssup_{x,y \in \domain}
\eta_{\tilde{\scale}_n}(\abs{x-y}) \normR{u_n(x) - u_n(y)} = 
\frac{\tilde{\scale}_n}{\scale_n}~\func_{\tilde{\scale}_n}(u_n).
\end{align*}
Together with the assumption $r_n/s_n\rightarrow 0$ we obtain that $\tilde{\scale}_n>0$ 
for $n$ large enough and $\tilde{\scale}_n/\scale_n \rightarrow 1$ which finally justifies 
the application of \cref{lem:NLLliminf}, i.e.,
\begin{align*}
\liminf_{n\rightarrow\infty} \funcd_{n,\mathrm{cons}}(u_n)
\geq\liminf_{n\rightarrow\infty}\frac{\tilde{\scale}_n}{\scale_n}~
\func_{\tilde{\scale}_n}(u_n)\geq
\sigma_{\eta}~\func(u)=\sigma_\eta~\func_\mathrm{cons}(u).
\end{align*}
\textbf{Case 2:} We now assume the kernel to fulfil \labelcref{en:K1}-\labelcref{en:K4}. 
The strategy is to find a $t>0$ where one can cut off the kernel without changing the 
value $\sigma_\eta$. From the continuity at $t=0$ \labelcref{en:K1} we have that 
\begin{align*}
\lim_{t\rightarrow 0} \eta(t)~t =0
\end{align*}
and thus there exists a $t^*>0$ such that 
\begin{align*}
\sup_{t\in [0,t^*]}\eta(t)~t\leq \sigma_\eta.
\end{align*}
We define 
\begin{align*}
\tilde{\eta}(t)=
\begin{cases}
\eta(t)&\text{ for } {t>t^*},\\
\eta(t^*)&\text{ for }t\in [0,t^*],
\end{cases}
\end{align*}
for which we have $\sigma_{\tilde{\eta}}=\sigma_\eta$ and thus the first case applies. 
Namely, using that $\eta$ is {non-increasing \labelcref{en:K2} and hence $\eta\geq\tilde\eta$} we obtain
{%
\begin{align*}
\liminf_{n\rightarrow\infty} \funcd_{n,\mathrm{cons}}(u_n)
&\geq
\liminf_{n\rightarrow\infty}\left(
\frac{1}{\scale_n}\esssup_{x,y \in \domain}
\tilde{\eta}_{\scale_n}(\abs{\vor_n(x)-\vor_n(y)}) \normR{u_n(x) - u_n(y)}\right)
\\&\geq
\sigma_{\eta}~\func_\mathrm{cons}(u).
\end{align*}}
\end{proof}
We now consider the limsup inequality for the constrained functionals.
\begin{lemma}[limsup inequality]\label{lem:LScons}
Let $\domain\subset\Rd$ be a domain satisfying~\labelcref{eq:cond_domain} and let the kernel fulfil 
\labelcref{en:K1}-\labelcref{en:K4}, then for a null sequence $\seq{s}{n}{(0,\infty)}$ 
and a function $u\in \Lp{\infty}(\Omega)$ there exists a sequence  
$\seq{u}{n}{\Lp{\infty}(\Omega)}$ converging to $u\in \Lp{\infty}(\Omega)$ in 
$\Lp{\infty}(\Omega)$ such that
\begin{align*}
\limsup_{n\rightarrow\infty}
\funcd_{n,\mathrm{cons}}(u_n) \leq \sigma_\eta~\func_{\mathrm{cons}}(u).
\end{align*}
\end{lemma}
\begin{proof}
If $\func_\mathrm{cons}(u)=\infty$ the inequality holds trivially. 
We thus consider $u\in W^{1,\infty}(\domain)$ such that $u(x)=g(x)$ for every $x\in\consSet$ and define a recovery sequence as follows: Let $\vecc u_n \in \Lp{\infty}(\nu_n)$ be defined by
\begin{align*}
    \vecc u_n(x) = 
    \begin{cases}
    u(x), \quad &x \in \domain_n \setminus \consSet_n, \\
    g(x), \quad &x \in \consSet_n,
    \end{cases}
\end{align*}
and define $u_n:=\vecc u_n \circ \vor_n$, where $p_n:\domain\to\domain_n$ denotes {a} closest point projection.
Then $u_n\in\Lp{\infty}(\domain)$ and by definition it holds 
\begin{align*}
\funcd_{n,\mathrm{cons}}(u_n)=\funcd_{n,\mathrm{cons}}(\vecc u_n) = \frac{1}{\scale_n} \max_{x,y\in\domain_n}\eta_{\scale_n}(\abs{x-y})\normR{\vecc u_n(x)-\vecc u_n(y)}.
\end{align*}
We have to distinguish three cases:\\
\textbf{Case 1:} Let $x,y\in\domain_n\setminus\consSet_n$, 
then we can compute, using \labelcref{eq:lipeq}
\begin{align*}
\abs{\vecc u_n(x)-\vecc u_n(y)}=
\normR{u(x)-u(y)}\leq 
d_\domain(x,y)~\norm{\gradR{u}}_{\Lp{\infty}}
\end{align*}
and therefore
\begin{align*}
\frac{1}{\scale_n}\eta_{\scale_n}(\abs{x-y})
\abs{\vecc u_n(x)-\vecc u_n(y)}
&\leq 
\underbrace{
\eta_{\scale_n}(\abs{x-y})~\frac{\normR{x-y}}{s_n}
}_{\leq \sigma_\eta}
\frac{d_\domain(x,y)}{\normR{x-y}}~
\norm{\gradR{u}}_{\Lp{\infty}}\\&\leq
\sigma_\eta
~\frac{d_\domain(x,y)}{\normR{x-y}}~
\norm{\gradR{u}}_{\Lp{\infty}}.
\end{align*}
\\
\textbf{Case 2:} Let $x\in\domain_n\setminus\consSet_n$ and $y\in\consSet_n$.
Then for every $\tilde y\in\consSet$ it holds, using~\labelcref{eq:lipeq}
\begin{align*}
\normR{\vecc u_n(x)-\vecc u_n(y)} 
&=  \normR{u(x)-g(y)} \\
&\leq  \normR{u(x)-u(\tilde y)} + 
\underbrace{\normR{u(\tilde y)-g(\tilde y)}}_{=0} + 
\normR{g(\tilde y)-g(y)}\\
&\leq \norm{\gradR{u}}_{\Lp{\infty}}\d_\domain(x,\tilde y) + 
\Lip(g)\normR{\tilde y-y} \\
&\leq\norm{\gradR{u}}_{\Lp{\infty}}d_\domain(x,y) + \norm{\gradR{u}}_{\Lp{\infty}}d_\domain(y,\tilde y) + 
\Lip(g) \normR{\tilde y-y} \\
&\leq \norm{\gradR{u}}_{\Lp{\infty}} d_\domain(x,y) +
\norm{\gradR{u}}_{\Lp{\infty}} \frac{d_\domain(y,\tilde y)}{\normR{y-\tilde y}} \normR{y-\tilde y}+ \Lip(g)\abs{y-\tilde y}.
\end{align*}
From this we have, using the same arguments as in the first case, 
that there is a $C>0$ such that
{
\begin{align*}
\frac{1}{\scale_n}\eta_{\scale_n}(\abs{x-y})
\abs{\vecc u_n(x)-\vecc u_n(y)}\leq
\sigma_\eta
~\frac{d_\domain(x,y)}{\normR{x-y}}~
\norm{\gradR{u}}_{\Lp{\infty}} + 
C~\frac{\normR{y-\tilde y}}{\scale_n}.
\end{align*}}
\\
\textbf{Case 3:} Let $x,y\in\consSet_n$, 
then for $\tilde x,\tilde y\in\consSet$ we have
\begin{align*}
\normR{\vecc u_n(x)-\vecc u_n(y)} 
&= \normR{g(x)-g(y)} \\
&= \normR{g(x)-u(\tilde x)} + \normR{u(\tilde x)-u(\tilde y)} + 
\normR{u(\tilde y)-g(y)}\\
&= \normR{g(x)-g(\tilde x)} + 
\normR{u(\tilde x)-u(\tilde y)} + \normR{g(\tilde y)-g(y)}\\
&\leq \Lip(g)\normR{x-\tilde x} + 
\norm{\gradR{u}}_{\Lp{\infty}} d_\domain(\tilde x,\tilde y) + 
\Lip(g)\normR{y-\tilde y}
\end{align*}
and therefore again
\begin{align*}
\frac{1}{\scale_n}\eta_{\scale_n}(\abs{x-y})
\abs{\vecc u_n(x)-\vecc u_n(y)}&\leq
\sigma_\eta
~\frac{d_\domain(x,y)}{\normR{x-y}}~
\norm{\gradR{u}}_{\Lp{\infty}} \\&+ 
\Lip(g)~\left(\frac{\normR{y-\tilde y}+\normR{x-\tilde x}}{\scale_n}\right).
\end{align*}
By \labelcref{eq:cond_domain} for every $\varepsilon>0$ there is $n_0\in\N$ sufficiently large such that for all $n\geq n_0$ it holds
\begin{align*} 
\sigma_\eta
~\frac{d_\domain(x,y)}{\normR{x-y}}~
\norm{\gradR{u}}_{\Lp{\infty}} \leq \sigma_\eta\norm{\gradR{u}}_{\Lp{\infty}} + 
\varepsilon/2,
\end{align*}
whenever $\abs{x-y}\leq \etaradius_\eta \scale_n$. 
Additionally, thanks to~\labelcref{eq:labelset_cvgc} and the compactness of $\consSet_n$ and~$\consSet$ for every $x\in\consSet_n$ we can choose $\tilde x\in\consSet$ such that 
\begin{align*}
\frac{\abs{x- \tilde x}}{\scale_n}\leq 
\frac{\varepsilon}{4 \max\{C,\Lip(g)\}}
\end{align*}
and analogously for $y$ and $\tilde y$. 
Combining the estimates from all three cases, we obtain
\begin{align*}
\funcd_{n,\mathrm{cons}}(u_n) 
&\leq \max\left\lbrace \sigma_\eta\norm{\gradR{u}}_{\Lp{\infty}} + \frac{\varepsilon}{2},\sigma_\eta\norm{\gradR{u}}_{\Lp{\infty}} + \frac{3\varepsilon}{4},\sigma_\eta\norm{\gradR{u}}_{\Lp{\infty}} + \varepsilon  \right\rbrace \\
&=\sigma_\eta~\norm{\nabla u}_{\Lp{\infty}} + \varepsilon
\end{align*}
for all $n\geq n_0$.
Finally, this yields
\begin{align*}
\limsup_{n\to\infty}\funcd_{n,\mathrm{cons}}(u_n) \leq \sigma_\eta~\func_\mathrm{cons}(u),
\end{align*}
as desired. 

For showing that $u_n\to u$ in $\Lp{\infty}(\Omega)$ one proceeds 
similarly:
If $p_n(x)\in\domain_n\setminus\consSet_n$ one has {thanks to \labelcref{eq:lipeq}
\begin{align*}
\normR{u(x)-u_n(x)}
&=\normR{u(x)-u(p_n(x))}\leq\norm{\gradR{u}}_{\Lp{\infty}}
d_\domain(x,p_n(x)) \\
&=\norm{\gradR{u}}_{\Lp{\infty}}\frac{d_\domain(x,p_n(x))}{\normR{x-p_n(x)}}\normR{x-p_n(x)} \to 0,\quad n\to\infty,
\end{align*}
where we also used \labelcref{eq:cond_domain}} and $\normR{x-p_n(x)}\leq r_n\to 0$.
In the case $p_n(x)\in\consSet_n$ {by \labelcref{eq:labelset_cvgc}} one again finds $\tilde x\in\consSet$ such that 
$\normR{p_n(x)-\tilde x}=o(\scale_n)$.
Then {by \labelcref{eq:lipeq,eq:cond_domain}} we have
\begin{align*}
\normR{u(x)-u_n(x)}&=\normR{u(x)-g(p_n(x))}\\
&\leq
\normR{u(x)-u(p_n(x))} + |u(p_n(x))-\underbrace{g(\tilde x)}_{=u(\tilde x)}|
+\normR{g(\tilde x)-g(p_n(x))}
\\&\leq
{
\norm{\gradR{u}}_{\Lp{\infty}}\frac{d_\domain(\vor_n(x),x)}{\normR{p_n(x)-x}}\normR{p_n(x)-x}} \\
&\qquad + \norm{\gradR{u}}_{\Lp{\infty}}\frac{d_\domain(p_n(x),\tilde x)}{\normR{p_n(x)-\tilde x}}\normR{p_n(x)-\tilde x} +\Lip(g) \normR{p_n(x)-\tilde x} \\
&\to 0
\end{align*}
since $\normR{p_n(x)-x}\leq r_n\to 0$ and 
$\normR{p_n(x)-\tilde x}=o(\scale_n)\to 0$.
Combining both cases proves $\norm{u-u_n}_{\Lp{\infty}}\to 0$ as $n\to\infty$.

\end{proof}
\begin{remark}{}{}
We note that the proof of the limsup inequality does not use any specific 
properties of the scaling, in fact even a sequence of disconnected graphs or the situation 
of \cref{ex:fuldisc} allows for such an inequality. 
\end{remark}
\begin{remark}[Relevance of the Hausdorff convergence]
The condition that the Hausdorff distance of $\consSet_n$ and $\consSet$ converges to zero as $n\to\infty$ (cf.~\labelcref{eq:labelset_cvgc}) implies both that $\consSet_n$ well approximates $\consSet$ and vice versa.
The first condition is only used in the proof of the liminf inequality \cref{lem:LIcons} whereas the second one only enters for the limsup inequality \cref{lem:LScons}.
Furthermore, the proof of the latter is drastically simplified if one assumes that $\consSet_n\subset\consSet$ for all $n\in\N$, which implies that the second term in the Hausdorff distance~\labelcref{eq:labelset_cvgc} equals zero.
{In this case, introducing the continuum points $\tilde x,\tilde y\in\consSet$ is not necessary and many estimates in the previous proof become trivial.}
\end{remark}
Combining \cref{lem:LIcons} and \cref{lem:LScons} we immediately obtain the $\Gamma$-convergence of the discrete functionals to those defined in the continuum, which is the statement of \cref{thm:DCGamma}.
\begin{remark}[Homogeneous boundary conditions]\label{rem:homogeneous_bdry}
In the case that $\consSet=\partial\domain$ and the constraints satisfy $g=0$ on $\consSet$ {any function with $\func_\mathrm{cons}(u)<\infty$ satisfies $u\in W^{1,\infty}_0(\domain)$. For this it is well-known that functions $u\in W^{1,\infty}_0(\domain)$ can be extended from} $\domain$ to $\Rd$ by zero without changing $\norm{\gradR{u}}_{\Lp{\infty}}$. 
In this case one can prove the limsup inequality \cref{lem:LScons} and hence also the $\Gamma$-convergence \cref{thm:DCGamma} for \emph{general open sets} $\domain$ without demanding \labelcref{eq:cond_domain} or even convexity.
{For this one simply utilizes the estimate 
\begin{align*}
    \abs{\vecc u_n(x)-\vecc u_n(y)}=\abs{u(x)-u(y)}\leq\norm{\nabla u}_{\Lp{\infty}}\abs{x-y}
\end{align*}
which is true if one extends $u$ by zero on $\R^d\setminus\domain$, multiplies with the kernel, and takes the supremum.}
\end{remark}

\section{Compactness}\label{sec:compactness}
We now want to make use of \cref{lem:ConvMin} in order to characterize the behaviour of 
minimizers of the discrete problems or more generally sequences of approximate minimizers, 
as described in the condition of the mentioned lemma.
The first result is a general 
characterization of relatively compact sets in $\Lp{\infty}$, the proof uses classical ideas 
from~\cite[Lem. IV.5.4]{Dunf60}.
{
\begin{lemma}\label{lem:compfirst}
Let $(\domain,\mu)$ be a finite measure space and $K\subset \Lp{\infty}(\domain;\mu)$ be a bounded set w.r.t. $\norm{\cdot}_{\Lp{\infty}(\domain;\mu)}$ such that 
for every $\varepsilon>0$ there exists a finite partition $\{V_i\}_{i=1}^n$ of $\domain$ into subsets $V_i$ with positive and finite measure such that
\begin{align}\label{eq:partinf}
\mu\operatorname{-}\esssup_{x,y\in V_i} \normR{u(x) - u(y)} < \varepsilon\ \forall u\in K, i=1,\ldots,n,
\end{align}
then $K$ is relatively compact. 
\end{lemma}
\newcommand*{\compop}{\mathcal{T}}
\begin{proof}
Let $\varepsilon>0$ be given and let $\{V_i\}_{i=1}^n$ be a partition into sets with finite and positive measure such that 
\begin{align}\label{eq:compopinf}
\mu\operatorname{-}\esssup_{x,y\in V_i} \normR{u(x) - u(y)} < \frac{\varepsilon}{3},\quad\ 
\forall u\in K,\; i=1,\ldots,n.
\end{align}
We define the operator 
$\compop:\Lp{\infty}(\domain;\mu)\rightarrow\Lp{\infty}(\domain;\mu)$ as
\begin{align*}
(\compop u)(x):= \frac{1}{\mu(V_i)}\int_{V_i} u(y)~\dd\mu(y)\quad\text{for } x\in V_i,
\end{align*}
which is well defined thanks to $0<\mu(V_i)<\infty$ for all $i=1,\dots,n$.
Using \labelcref{eq:compopinf} we observe that for $\mu$-almost every $x\in V_i$
\begin{align*}
\normR{u(x) - (\compop u)(x)} \leq \frac{1}{\mu(V_i)}\int_{V_i} \normR{u(x)-u(y)}~\dd\mu(y)
< \frac{\varepsilon}{3}
\end{align*}
and thus $\norm{u - \compop u}_{\Lp{\infty}(\domain;\mu)}< \frac{\varepsilon}{3}$.
Furthermore, $\compop(K)\subset\text{span}(\{\ind_{V_1},\ldots, \ind_{V_n}\})$,
where we let $\ind_M$ denote the indicator function of a set $M$, defined by $\ind_M(x)=0$ if $x\notin M$ and $\ind_M(x)=1$ if $x\in M$.
Hence, $\compop$ has finite-dimensional range and since $K$ is bounded we have
\begin{align*}
\norm{\compop u}_{\Lp{\infty}(\domain;\mu)} \leq \norm{u}_{\Lp{\infty}(\domain;\mu)} \leq C\quad \forall u\in K
\end{align*}
and therefore $\compop(K)$ is relatively compact. 
This implies that there exist finitely many 
functions $\{u_j\}_{j=1}^{N}\subset \Lp{\infty}(\domain;\mu)$ such that 
\begin{align*}
\compop(K)\subset \bigcup_{j=1}^N B_{\frac{\varepsilon}{3}}(\compop(u_j)),
\end{align*}
where $B_t(u):=\{v\in\Lp{\infty}(\domain;\mu)\;:\;\norm{u-v}_{\Lp{\infty}(\domain;\mu)}<t\}$ denotes the open ball with radius $t>0$ around $u\in\Lp{\infty}(\domain;\mu)$.
For $u\in K$ we can thus find 
$j\in\{1,\ldots,N\}$ such that $\compop(u)\in B_{\frac{\varepsilon}{3}}(\compop(u_j))$ 
and thus 
\begin{align*}
\norm{u-u_j}_{\Lp{\infty}} \leq \norm{u - \compop(u)}_{\Lp{\infty}} + 
\norm{\compop(u) - \compop(u_j)}_{\Lp{\infty}} + 
\norm{\compop(u_j)-u_j}_{\Lp{\infty}} < \eps.
\end{align*}
This implies that $K$ is totally bounded and since $\Lp{\infty}(\domain;\mu)$ is complete 
the result follows from \cite[Lem. I.6.15]{Dunf60}.
\end{proof}}
The previous lemma allows us to prove a compactness result for the non-local functionals, 
where we again need the domain $\domain$ to fulfill condition~\labelcref{eq:cond_domain}.
\begin{lemma}\label{lem:compnonloc}
Let $\domain\subset\Rd$ be a bounded domain satisfying~\labelcref{eq:cond_domain} and 
let the kernel $\eta$ fulfil 
\labelcref{en:K1}-\labelcref{en:K4}, and let $\seq{\scale}{n}{(0,\infty)}$ be a null sequence. 
Then every bounded sequence $\seq{u}{n}{\Lp{\infty}(\domain)}$ such that
\begin{align}
\sup_{n\in\N}\func_{\scale_n}(u_n)&<\infty\label{eq:compnonloc}
\end{align}
is relatively compact.
\end{lemma}
\begin{proof}
We want to apply \cref{lem:compfirst} in order to see that the sequence is 
relatively compact. Therefore let $\varepsilon>0$ be given and w.l.o.g. 
we rescale the kernel such that
\begin{align*}
\eta(t)\geq 1\text{ for }t\leq 1.
\end{align*}
Using \labelcref{eq:cond_domain} we can find $\delta>0$ such that 
for every $x,y\in\domain$ with $\abs{x-y}\leq\delta$ there is 
a path $\gamma:[0,1]\rightarrow\domain$ such that 
$\gamma(0)=x, \gamma(1)=y$ and
\begin{align*}
\text{len}(\gamma)\leq (1+\varepsilon)\normR{x-y}.
\end{align*}
We divide this path by points $0=t_0<\ldots<t_i<\ldots<t_{k_n}=1$
such that for
$z_i:=\gamma(t_i)$ we have that 
\begin{align*}
\normR{z_i-z_{i+1}}\leq \scale_n
\end{align*}
 for $i=0,\ldots, k_n$, where
\begin{align*}
k_n \leq \lfloor (1+\varepsilon)~\normR{x-y}/s_n \rfloor.
\end{align*}
Then we have that 
\begin{align*}
\normR{u_n(x)-u_n(y)} &\leq \sum_{i=0}^{{k_n-1}} {\normR{u_n(z_i) - u_n(z_{i+1})}}
\\&\leq 
s_n~\sum_{i=0}^{{k_n-1}} \eta_{s_n}(\abs{z_i-z_{i+1}}) \frac{\normR{u_n(z_i) - u_n(z_{i+1})}}{s_n}\\ &\leq
s_n~\sum_{i=0}^{{k_n-1}} \func_{s_n}(u_n)\\&\leq
s_n~k_n~\underbrace{\sup_{n\in\N}\func_{\scale_n}(u_n)}_{=:C<\infty}\\&\leq 
C~(1+\varepsilon)~ \normR{x-y}.
\end{align*}
Choosing a partition $\{V_i\}_{i=1}^N$ of $\domain$ {into sets with positive Lebesgue measure such that}
\begin{align*}
\text{diam}(V_i)< \min\left\{\delta, \frac{\varepsilon}{C~(1+\varepsilon)}\right\}    
\end{align*}
for $i=1,\ldots,N,$ yields that 
\begin{align*}
\esssup_{x,y\in V_i} \normR{u_n(x)- u_n(y)}\leq 
C~(1+\varepsilon)~ \esssup_{x,y\in V_i} \normR{x-y}\leq
C~(1+\varepsilon)~\text{diam}(V_i) < \varepsilon.
\end{align*}
Since $\seq{u}{n}{\Lp{\infty}}$ is bounded in $\Lp{\infty}$ we can therefore apply
\cref{lem:compfirst} to infer that the sequence is relatively compact.
\end{proof}
{We will use this result in order to prove that the constrained functionals $\funcd_{n,\mathrm{cons}}$ are compact, which then directly shows \cref{thm:ConvMinLip}.}
The intuitive reason {that these functionals are compact} is the fact that for a domain $\domain$ that fulfills \labelcref{eq:cond_domain} 
each point $x\in\domain_n$ has finite geodesic distance to the set $\consSet_n$. 
This follows from the fact that the geodesic diameter of $\domain$ is bounded, as we show in the following lemma.
\begin{lemma}\label{rem:bounded_diameter}
Condition \labelcref{eq:cond_domain} implies
that the geodesic diameter is finite, i.e., 
\begin{align}\label{eq:geodesic_diameter}
\diam_g(\domain):=\sup_{x,y\in\domain}d_\domain(x,y) < \infty.
\end{align}
\end{lemma}
\begin{proof}
For $\varepsilon>0$ we can {use \labelcref{eq:cond_domain} to} find $\delta>0$ such that 
$d_\domain(x,y)<\abs{x-y}(1+\varepsilon)$ for all $x,y\in\domain$ with $\normR{x-y}<\delta$. 
Since we assume $\domain$ to be bounded we know that there exists a finite collection
$\{x_1,\ldots,x_N\}\subset\domain$
such that 
\begin{align*}
\domain\subset\bigcup_{i=1}^N B_\delta(x_i).
\end{align*}
If two balls at centers $x_i,x_j$ share a common point $z\in\domain$ we see that 
\begin{gather}\label{eq:balldiameter}
\begin{aligned}
d_\domain(x_i,x_j)&\leq d_\domain(x_i,z)+d_\domain(z,x_j)
\\&\leq (1+\varepsilon)\abs{x_i - z} + (1+\varepsilon)\abs{z - x_j}
\\&\leq 2(1+\varepsilon)\delta.
\end{aligned}
\end{gather}
For any $x,y\in\domain$ assume that there exists a path $\gamma$ in $\domain$ 
from $x$ to $y$. Therefore, also the image of $\gamma$ is covered by finitely 
many balls at centers $x_{k_1},\ldots,x_{k_n}$ such that 
\begin{align*}
B_\delta(x_{k_i})\cap B_\delta(x_{k_{i+1}})\cap \domain \neq\emptyset
\end{align*}
for $i=1,\ldots,{n-1}$ with $x\in B_\delta(x_{k_1}), y\in B_\delta(x_{k_n})$. 
Using \labelcref{eq:balldiameter} this yields 
\begin{align*}
d_\domain(x,y)&\leq 
d_\domain(x,x_{k_1}) + \sum_{i=1}^{k_{n-1}} d_\domain(x_i,x_{i+1}) +
d_\domain(x_{k_n},y)\\&\leq
2~(N + 1)(1+\varepsilon)\delta
\end{align*}
{where we used $k_n\leq N$.}
Note that the last expression above is independent of $x,y$ {which concludes the proof.}
\end{proof}
\begin{lemma}\label{boundseql}
Let $\domain\subset\Rd$ be a domain satisfying~\labelcref{eq:cond_domain}, let the kernel fulfil 
\labelcref{en:K1}-\labelcref{en:K4}, and $\seq{s}{n}{(0,\infty)}$ be a null sequence which satisfies the scaling condition~\labelcref{eq:scaling}.
Let $\seq{u}{n}{\Lp{\infty}}(\domain)$ be a sequence with
\begin{align*}
\sup_{n\in\N} \funcd_{n,\mathrm{cons}}(u_n) &<\infty
\end{align*} 
then it is bounded with respect to $\norm{\cdot}_\infty$.
\end{lemma}
\begin{remark}
{Similar to \cref{rem:homogeneous_bdry}}, one can relax condition \labelcref{eq:cond_domain} in \cref{boundseql} by taking into account the specific form of the constraint set $\consSet\subset\overline{\domain}$.
Indeed, {an inspection of the following proof shows that} it suffices to demand that
\begin{align}\label{eq:relaxation_geo_diam}
    \sup_{x\in\domain}\inf_{y\in\consSet} d_\domain(x,y)<\infty,
\end{align}
which means that $\consSet$ has finite geodesic distance to any point in $\domain$. 
In the case that $\consSet=\partial\domain$ this is always satisfied since $\domain$ is bounded. 
However, the condition is violated, e.g., if {$\domain$ is an open, infinite but bounded spiral and $\consSet$ a single point at its center.}
\end{remark}
\begin{proof}
W.l.o.g. we assume 
\begin{align*}
\eta(t)\geq 1\text{ for }t\leq 1.
\end{align*}
Let $n_0\in\N$ be large enough such that 
\begin{align*}
r_n < \scale_n/2,\ \forall n\geq n_0
\end{align*}
and for $x_0\in\domain_n, n\geq n_0$ let $\gamma:[0,1]\rightarrow\domain$ be a path in $\domain$ such 
that $\gamma(0)=x_0$ and $\gamma(1)\in\consSet_n.$ We divide $\gamma$ by points {$0=t_0<\ldots<t_{k_n}=1$ such that 
\begin{align*}
\normR{\gamma(t_i) - \gamma(t_{i+1})}&= \scale_n - 2r_n,\quad i=0,\ldots,k_n-2,\\
\normR{\gamma(t_{k_n}) - \gamma(t_{k_n+1})}&\leq \scale_n - 2r_n,
\end{align*}
and by definition of the parameter $r_n$ we know that for each $i=0,\ldots,k_n$} there exists a vertex $x_i\in\domain_n$ such that
\begin{align*}
\normR{x_i - \gamma(t_i)}\leq r_n.
\end{align*}
Applying the triangle inequality this yields
\begin{align*}
\normR{x_i - x_{i+1}}&\leq \normR{x_i- \gamma(t_i)} + 
\normR{\gamma(t_i) - \gamma(t_{i+1})} + \normR{x_{i+1} - \gamma(t_{i+1})}\\&\leq 
2r_n + \scale_n - 2 r_n\leq s_n, 
\end{align*}
and thus {$\eta_{\scale_n}(\abs{x_i-x_{i+1}})\geq 1$ for all $i=0,\dots,k_n-1$.
By definition of the discrete functional \labelcref{eq:TrafoDiscConstr} there exists $\vecc u_n:\domain\to\R$ with $u_n=\vecc u_n\circ\vor_n$ and} we can estimate
\begin{align}
\normR{\vecc u_n(x_0)}&\leq \sum_{i=0}^{{k_n-2}} 
\normR{\vecc u_n(x_i) - \vecc u_n(x_{i+1})} + \normR{\vecc u_n(x_{{k_n}})}\nonumber
\\&\leq
\scale_n \sum_{i=0}^{{k_n-2}} \eta_{\scale_n}(\abs{x_i-x_{i+1}}) 
\normR{\vecc u_n(x_i) - \vecc u_n(x_{i+1})}/\scale_n
+ \normR{g(x_{k_n})}\nonumber
\\&\leq
\scale_n~{(k_n-1)}~\funcd_{n, \mathrm{cons}}(u_n) + \normR{g(x_{k_n})},\label{eq:lenest}
\end{align}
{where we used $\vecc u_n(x)=g(x)$ for all $x\in\consSet_n$ since $\funcd_{n,\mathrm{cons}}(u_n)<\infty$.}
It remains to show that the product $\scale_n~{(k_n-1)}$ is uniformly bounded in $n$, for which 
we first observe that the path $\gamma$ can be chosen such that
\begin{align*}
k_n-1\leq \lfloor \text{diam}_g(\domain)/(\scale_n-2r_n) \rfloor
\end{align*}
and thus using~\labelcref{eq:scaling}
{
\begin{align*}
\scale_n~(k_n-1) \leq
\scale_n~\left\lfloor \frac{\text{diam}_g(\domain)}{(\scale_n-2r_n)}\right\rfloor 
\leq C \frac{1}{(1 - 2~r_n/\scale_n)} < \tilde C,\ \forall n\in\N,
\end{align*}}
where we note that $\diam_g(\domain)<\infty$ according to \cref{rem:bounded_diameter}.
Together with \labelcref{eq:lenest} this yields that there exists a uniform constant 
$C>0$ such that $\norm{u_n}_{\Lp{\infty}}\leq C$ for all $n\in\N$.
\end{proof}
We can now prove {that the constrained functionals $\funcd_{n,\mathrm{cons}}$ are indeed compact.}
\begin{lemma}\label{lem:compdiscseq}
Let $\domain\subset\Rd$ be a domain satisfying~\labelcref{eq:cond_domain}, let the kernel fulfil \labelcref{en:K1}-\labelcref{en:K4}, and $\seq{s}{n}{(0,\infty)}$ be a null sequence which satisfies the scaling condition~\labelcref{eq:scaling}.
Then we have that every sequence $\seq{u}{n}{\Lp{\infty}(\Omega)}$ such that
\begin{align*}
\sup_{n\in\N} \funcd_{n,\mathrm{cons}}(u_n)<\infty
\end{align*}
is relatively compact in $\Lp{\infty}$.
\end{lemma}
\begin{proof}
Using the same arguments as in the proof of \cref{lem:LIcons} we can find a 
scaling sequence $\seq{\tilde\scale}{n}{(0,\infty)}$ such that
\begin{align*}
\funcd_{n}(u_n) &\geq 
\frac{\tilde{\scale}_n}{\scale_n}~\func_{\tilde{\scale}_n}(u_n)
\end{align*}
and $\tilde{\scale}_n/\scale_n\longrightarrow 1$. We choose 
$C:=\sup_{n\in\N} \frac{\scale_n}{\tilde{\scale}_n}<\infty$ to obtain 
\begin{align*}
\sup_{n\in\N}\func_{\tilde{\scale}_n}(u_n) \leq C\sup_{n\in\N}\funcd_{n}(u_n) = C\sup_{n\in\N}\funcd_{n,\mathrm{cons}}(u_n) <\infty.
\end{align*}
Thanks to \cref{boundseql} the sequence $(u_n)_{n\in\N}$ is bounded in $\Lp{\infty}$ and thus we can apply \cref{lem:compnonloc} to infer that 
$\seq{u}{n}{\Lp{\infty}}(\domain)$ is relatively compact.
\end{proof}
Together with \cref{lem:ConvMin} this finally yields our second 
main statement \cref{thm:ConvMinLip}.
{%
\begin{proof}[Proof of \cref{thm:ConvMinLip}]
Let $v\in\Lp{\infty}(\Omega)$ with $\func_\mathrm{cons}(v)<\infty$ be arbitrary and $\seq{v}{n}{\Lp{\infty}(\Omega)}$ be a recovery sequence for $v$, the existence of which is guaranteed by \cref{lem:LScons}.
By assumption, the sequence $u_n$ satisfies
\begin{align*}
    \limsup_{n\to\infty}\funcd_{n,\mathrm{cons}}(u_n) &= 
    \limsup_{n\to\infty}
    \inf_{u\in\Lp{\infty}(\Omega)} \funcd_{n,\mathrm{cons}}(u)
    \\
    &\leq
    \limsup_{n\to\infty}
    \funcd_{n,\mathrm{cons}}(v_n)
    \leq
    \sigma_\eta
    ~\func_\mathrm{cons}(v)<\infty.
\end{align*}
Hence, \cref{lem:compdiscseq} implies that $\seq{u}{n}{\Lp{\infty}(\Omega)}$ is relatively compact.
\cref{lem:ConvMin} then concludes the proof.
\end{proof}}
\section{Application to Ground States}\label{sec:groundstates}
In this section we apply the discrete-to-continuum $\Gamma$-convergence from \cref{thm:DCGamma} to so-called ground states, first studied in \cite{bung20}.
These are restricted minimizers of the functionals $\funcd_{n,\mathrm{cons}}$ and $\func_\mathrm{cons}$ on $\Lp{p}$-spheres, where we assume that the constraint satisfies $g=0$ on $\overline{\domain}$.
This makes the functionals $\funcd_{n,\mathrm{cons}}$ and $\func_\mathrm{cons}$ absolutely $1$-homogeneous.

For absolutely $p$-homogeneous functionals $F:X\to\R\cup\{\infty\}$ on a Banach space {$(X,\norm{\cdot})$} with $p\in[1,\infty)$, which per definitionem satisfy
\begin{align*}
    F(cu)=\normR{c}^pF(u),\quad\forall u\in X,\,c\in\R,
\end{align*}
ground states are defined as solutions to the minimization problem
\begin{align*}
    \min\left\lbrace F(u) \,:\, \inf_{v\in\argmin F}\norm{u-v} = 1 \right\rbrace.
\end{align*}
Ground states and their relations to gradient flows and power methods are well-studied in the literature, see, e.g., \cite{hynd2016inverse,hynd2017approximation,hynd2019extremal,bungert2019asymptotic,feld2019rayleigh,bung20}.
In particular, they constitute minimizers of the non-linear Rayleigh quotient
\begin{align*}
    R(u) = \frac{F(u)}{\inf_{v\in\argmin F}\norm{u-v}^p}
\end{align*}
and are related to non-linear eigenvalue problems with the prime example being $F(u)=\int_\domain\normR{\gradR{u}}^p\d x$ {on $X:=\Lp{\domain}$} where ground states solve the $p$-Laplacian eigenvalue problem
\begin{align*}
    \lambda\normR{u}^{p-2}u = -\Delta_p u.
\end{align*}
In \cite{bung20} ground states of the functionals $\funcd_{n,\mathrm{cons}}$ and $\func_\mathrm{cons}$ were characterized as distance functions. 
While there it was assumed that $\consSet=\partial\domain$, we will in the following generalize these results to the case of an arbitrary closed constraint set $\consSet\subset\overline{\domain}$.
Subsequently, we will use the $\Gamma$-convergence, established in \cref{thm:DCGamma}, to show discrete-to-continuum convergence of ground states.

\subsection{Relation to Distance Functions}

Here, we show that the unique $\Lp{p}$ ground states of the limit functional $\func_\mathrm{cons}$ coincide with multiples of the geodesic distance function to the set $\consSet$.
To prove the desired statement, we need the following lemma, stating that the gradient of the geodesic distance function is bounded by one.

\begin{lemma}\label{lem:grad_dist}
Let $\consSet\subset\overline{\domain}$ be a closed set and 
\begin{align*}
    d_\consSet(x):=\inf_{y\in\consSet}d_\domain(x,y)
\end{align*}
be the geodesic distance function of $\consSet$, where $d_\domain(x,y)$ denotes the geodesic distance between $x,y\in\domain$.
Then it holds
\begin{align*}
    \norm{\gradR{d_\consSet}}_{\Lp{\infty}}\leq 1.
\end{align*}
\end{lemma}
\begin{proof}
Let $x,y\in\domain$ be arbitrary.
Using the triangle inequality for $d_\domain$ we get
\begin{align*}
    d_\consSet(y)\leq d_\domain(x,y)+d_\consSet(x),\quad\forall x,y\in\domain.
\end{align*}
{If $x,y$ lie in a ball fully contained in $\Omega$, then $d_\Omega(x,y)=\abs{x-y}$ and we obtain that $d_\consSet$ is Lipschitz continuous on this ball. 
Rademacher's theorem then implies that $\nabla d_\consSet$ exists almost eveywhere in the ball.
Since the ball is arbitrary, $\nabla d_\consSet$ in fact exists almost everywhere in $\domain$.}

{Furthermore, since $\domain$ is open, for $x\in\domain$ and} $t>0$ small enough the ball $B_t(x):=\{y\in\Rd\,:\,\normR{x-y}<t\}$ lies within $\domain$ and it holds $d_\domain(x,y)=\normR{x-y}$ for all $y\in B_t(x)$. 

Choosing $y=x+ta\in B_t(x)$ with $a\in B_1(0)$ we get
\begin{align*}
    \frac{d_\consSet(x+ta)-d_\consSet(x)}{t}\leq \frac{d_\domain(x,x+ta)}{t}= \frac{\normR{at}}{t}\leq 1.
\end{align*}
Since $a$ was arbitrary, we can conclude $\normR{\gradR{d_\consSet}(x)}\leq 1$ for almost all $x\in\domain$ which implies the desired statement.
\end{proof}

With this lemma we now can prove that the unique ground state (up to scalar multiples) of the functional $\func_\mathrm{cons}$ is given by the geodesic distance function to $\consSet$.
The only (weak) assumption which we need here is that $d_\consSet\in\Lp{p}(\domain)$ which is fulfilled, for instance, if $\domain$ has finite geodesic diameter or even only satisfies the relaxed condition~\labelcref{eq:relaxation_geo_diam}, in which case $d_\consSet\in\Lp{\infty}(\domain)$ holds.

\begin{theorem}\label{thm:GSDist}
Let $\consSet$ be {a closed set} such that $\overline{\domain}\setminus\consSet$ is connected {and non-empty} and $d_\consSet\in\Lp{p}(\domain)$, and let the constraint function satisfy $g=0$ on $\consSet$.
For $p\in[1,\infty)$ the unique solution (up to global sign) to
\begin{align}\label{eq:ground_state}
    \min\left\lbrace
    \func_{\mathrm{cons}}(u) \,:\,u\in\Lp{\infty}(\domain),\, \norm{u}_{\Lp{p}}= 1 \right\rbrace
\end{align}
is given by a positive multiple of the geodesic distance function $d_\consSet$.

If $\domain$ is convex or $\consSet=\partial\domain$, the geodesic distance $d_\domain(x,y)$ {in the definition of $d_\consSet$} can be replaced by the Euclidean $|x-y|$ and $d_\consSet\in\Lp{p}(\domain)$ is always satisfied.
\end{theorem}
\begin{proof}
The case $\consSet=\partial\domain$ was already proved in \cite{bung20}.
{If $\Omega$ is convex it holds $d_\domain(x,y)=\abs{x-y}$ which is and hence $d_\consSet$ is bounded and in particular lies in $\Lp{p}(\domain)$ for all $p\geq 1$.}

We {first prove} that the geodesic distance function $d_\consSet$ is a solution of
\begin{align}\label{eq:ground_state_max}
    \max\left\lbrace \norm{u}_{\Lp{p}} \,:\,{u\in\Lp{\infty}(\domain)},\, \func_\mathrm{cons}(u)= 1 \right\rbrace.
\end{align}
Since $\consSet$ is closed and bounded, for every $x\in\domain$ we can choose $y_x\in\consSet$ such that $d_\domain(x,y_x)\leq d_\domain(x,y)$ for all $y\in\consSet$.
Hence, if $u=0$ on $\consSet$, we can choose $y=y_x$ and obtain {from \labelcref{eq:lipeq} that}
\begin{align}\label{ineq:pw_maximality_dist_func}
    \normR{u(x)} \leq \norm{\gradR{u}}_{\Lp{\infty}}d_\consSet(x)
\end{align}
{for almost every $x\in\domain$} which implies
\begin{align}\label{ineq:maximality_dist_func}
    \norm{u}_{\Lp{p}}\leq\norm{\gradR{u}}_{\Lp{\infty}}\norm{d_\consSet}_{\Lp{p}}.
\end{align}
Hence, for all $u\in\Lp{\infty}(\domain)$ with $\func_\mathrm{cons}(u)= 1$ one obtains from \labelcref{ineq:maximality_dist_func} that
\begin{align*}
    \norm{u}_{\Lp{p}}\leq\norm{d_\consSet}_{\Lp{p}}.
\end{align*}
{From \cref{lem:grad_dist} we know that} $\func_\mathrm{cons}(d_\consSet)=\norm{\gradR{d_\consSet}}_{\Lp{\infty}}\leq 1$.
{At the same time choosing $u=d_\consSet$ in \labelcref{ineq:maximality_dist_func} shows that in fact $\func_\mathrm{cons}(d_\consSet)=\norm{\gradR{d_\consSet}}_{\Lp{\infty}}= 1$ and therefore} $d_\consSet$ solves \labelcref{eq:ground_state_max}.

Regarding uniqueness we argue as follows: {Since} $p<\infty$ the inequality \labelcref{ineq:maximality_dist_func} is sharp if \labelcref{ineq:pw_maximality_dist_func} is sharp which{, using that $\norm{\nabla u}_{\Lp{\infty}}=\func_\mathrm{cons}(u)=1$, } implies that all solutions $u$ of \labelcref{eq:ground_state_max} must fulfill 
\begin{align*}
    \normR{u(x)}=d_\consSet(x),\quad\forall x\in\domain.
\end{align*}
If $\domain\setminus\consSet$ is connected, the continuity of $u$ implies that (up to global sign) $u(x)=d_\consSet(x)$ for all $x\in\domain$.

{Finally, we argue that \labelcref{eq:ground_state,eq:ground_state_max} are equivalent:
Since $d_\consSet\in\Lp{p}(\domain)$ we have $\norm{d_\consSet}_{\Lp{p}}<\infty$.
Then $d_\consSet/\norm{d_\consSet}_{\Lp{p}}$ solves \labelcref{eq:ground_state} since for any $u\in\Lp{\infty}(\domain)$ with $\norm{u}_{\Lp{p}}=1$ it holds
\begin{align*}
    \func_\mathrm{cons}\left(\frac{d_\consSet}{\norm{d_\consSet}_{\Lp{p}}}\right) 
    =
    \frac{\func_\mathrm{cons}(d_\consSet)}{\norm{d_\consSet}_{\Lp{p}}} 
    =
    \frac{1}{\norm{d_\consSet}_{\Lp{p}}}
    =
    \frac{\norm{u}_{\Lp{p}}}{\norm{d_\consSet}_{\Lp{p}}}
    \leq
    \func_\mathrm{cons}(u),
\end{align*}
where we used \labelcref{ineq:maximality_dist_func} for the inequality.
Analogously, if $u$ solves \labelcref{eq:ground_state} then 
\begin{align*}
    \func_\mathrm{cons}(u)\leq \func_\mathrm{cons}(d_\consSet/\norm{d_\consSet}_{\Lp{p}})=1/\norm{d_\consSet}_{\Lp{p}}<\infty.
\end{align*}
This follows from the fact that $d_\consSet\neq 0$ since $\overline\domain\setminus\consSet\neq\emptyset$.
Then $u/\func_\mathrm{cons}(u)$ solves \labelcref{eq:ground_state_max} since, using again \labelcref{ineq:maximality_dist_func}, it holds
\begin{align*}
    \norm{\frac{u}{\func_\mathrm{cons}(u)}}_{\Lp{p}} = \frac{1}{\func_\mathrm{cons}(u)}
    \geq 
    \frac{\norm{d_\consSet}_{\Lp{p}}}{\norm{u}_{\Lp{p}}} = \norm{d_\consSet}_{\Lp{p}}
\end{align*}
and $d_\consSet$ solves \labelcref{eq:ground_state_max}.
}
\end{proof}

\begin{remark}
In the case $p=\infty$ the geodesic distance function is still a ground state, however, not the unique one.
In this case, other ground states are given by $\infty$-Laplacian eigenfunctions, see, e.g.,~\cite{juutinen1999infinity, yu2007some,bozorgnia2020infinity}.
\end{remark}

\begin{remark}
If one drops the condition that $\overline{\domain}\setminus\consSet$ is connected, ground states coincide with (positive or negative) multiples of the distance function on each connected component of $\overline{\domain}\setminus\consSet$.
\end{remark}

Similarly, one can also prove that ground states of the discrete functionals $\funcd_{n,\mathrm{cons}}$ coincide with multiples of distance functions to $\consSet_n$ with respect to the geodesic graph distance if $\domain_n\setminus\consSet_n$ is connected in the graph-sense.
The result can be found in \cite{bung20}, however, since we do not need it here, we refrain from stating it.

\subsection{Convergence of Ground States}

In this section we first show that the $\Gamma$-convergence of the functionals $\funcd_{n,\mathrm{cons}}$ to $\sigma_\eta\func_\mathrm{cons}$ implies the convergence of their respective ground states.
Together with the characterization from \cref{thm:GSDist} this implies that discrete ground states converge to the geodesic distance function.

\begin{theorem}[Convergence of Ground States]\label{thm:cgvc_GS}
Under the conditions of \cref{thm:DCGamma} let the sequence $(u_n)_{n\in\N}\subset\Lp{\infty}(\domain)$ fulfill
\begin{align*}
    u_n \in \argmin\left\lbrace
    \funcd_{n,\mathrm{cons}}(u) \,:\,u\in\Lp{\infty}(\domain),\, \norm{u}_{\Lp{p}}=1 \right\rbrace.
\end{align*}
Then (up to a subsequence) $u_n\to u$ in $\Lp{\infty}(\domain)$ where
\begin{align*}
    u \in \argmin\left\lbrace
    \func_{\mathrm{cons}}(u) \,:\,u\in\Lp{\infty}(\domain),\, \norm{u}_{\Lp{p}}=1 \right\rbrace
\end{align*}
and it holds
\begin{align}\label{eq:cvgc_eigenvalues}
    \lim_{n\to\infty}{\funcd_{n,\mathrm{cons}}(u_n)} = \sigma_\eta~\func_{\mathrm{cons}}(u).
\end{align}
\end{theorem}
\begin{proof}
Let $u\in\Lp{\infty}(\domain)$ be a ground state of $\func_\mathrm{cons}$ and $(v_n)_{n\in\N}\subset\Lp{\infty}(\domain)$ be a recovery sequence of $u$, whose existence is guaranteed by \cref{thm:DCGamma}.
Since $u_n$ is a ground state and $\funcd_{n,\mathrm{cons}}$ is absolutely $1$-homogeneous, we get
\begin{align*}
    \funcd_{n,\mathrm{cons}}(u_n)
    \leq \funcd_{n,\mathrm{cons}}\left(\frac{v_n}{\norm{v_n}_{\Lp{p}}}\right) 
    = \funcd_{n,\mathrm{cons}}(v_n)\frac{1}{\norm{v_n}_{\Lp{p}}}.
\end{align*}
Taking the limsup on both sides yields
\begin{align*}
    \limsup_{n\to\infty} \funcd_{n,\mathrm{cons}}(u_n) \leq \sigma_\eta~\func_\mathrm{cons}(u)\frac{1}{\norm{u}_{\Lp{p}}} = \sigma_\eta~\func_\mathrm{cons}(u) < \infty,
\end{align*}
{where we used boundedness of $\domain$ to conclude that $\Lp{\infty}$-convergence implies convergence of the $\Lp{p}$-norms.}
Hence, by \cref{lem:compdiscseq} the sequence $(u_n)_{n\in\N}$ posseses a subsequence (which we do not relabel) which converges to some $u^*\in\Lp{\infty}(\domain)$ with $\norm{u^*}_{\Lp{p}}=1$.
Using the previous inequality, the liminf inequality from \cref{lem:LIcons}, and the fact that $u$ is a ground state we conclude
\begin{align*}
    \sigma_\eta~\func_\mathrm{cons}(u^*)
    &\leq \liminf_{n\to\infty} \funcd_{n,\mathrm{cons}}(u_n)  \\
    &\leq \limsup_{n\to\infty}  \funcd_{n,\mathrm{cons}}(u_n) \\
    &\leq \sigma_\eta~\func_\mathrm{cons}(u) \\
    &\leq \sigma_\eta~\func_\mathrm{cons}(u^*).
\end{align*}
Hence, $u^*$ is also a ground state and \labelcref{eq:cvgc_eigenvalues} holds true.
\end{proof}

Using the characterization of ground states as distance functions we obtain the following
\begin{corollary}
Under the conditions of \cref{thm:GSDist} and \cref{thm:cgvc_GS} the sequence $(u_n)_{n\in\N}\subset\Lp{\infty}(\domain)$, given by 
\begin{align*}
    u_n \in \argmin\left\lbrace
    \funcd_{n,\mathrm{cons}}(u) \,:\,u\in\Lp{\infty}(\domain),\, \norm{u}_{\Lp{p}}=1 \right\rbrace,
\end{align*}
converges to a multiple of the geodesic distance function $d_\consSet$.
\end{corollary}

\section{Conclusion and Future Work}\label{sec:conclusion}
In this work we derived continuum limits of semi-supervised 
Lipschitz learning on graphs.
We first proved $\Gamma$-convergence of non-local 
functionals to the supremal norm of the gradient. 
This allowed us to show $\Gamma$-convergence of the discrete energies 
which appear in the Lipschitz learning problem.
In order to interpret graph functions as functions defined on the 
continuum, we employed a closest point projection.
We also showed that the discrete functionals are compact which implies 
discrete-to-continuum convergence of minimizers.
We applied our results to a nonlinear eigenvalue problem whose solutions are 
geodesic distance functions.

Future work will include the generalization of our results to general 
metric measure spaces or Riemannian manifolds, which constitute a 
generic domain for the data in real-world semi-supervised learning problems.
Furthermore, we intend to see how the application of our results 
to absolutely minimizing Lipschitz extensions~{\cite{aronsson2004tour}} on graphs unfolds. 
Namely, we want to gain insight whether it is possible to prove 
their convergence towards solutions of the infinity Laplacian equation 
under less restrictive assumptions than the ones used in \cite{calder2020poisson}.

\section*{Acknowledgments}
This work was supported by the European Unions Horizon 2020 research and innovation
programme under the Marie Sk{\l}odowska-Curie grant agreement No 777826 (NoMADS). 
The work of TR was supported by the German 
Ministry of Science and Technology (BMBF) under
grant 05M2020 - DELETO.
LB acknowledges funding by the Deutsche Forschungsgemeinschaft (DFG, German Research Foundation) under Germany's Excellence Strategy - GZ 2047/1, Projekt-ID 390685813.

\printbibliography
\end{document}